\PassOptionsToPackage{numbers, compress}{natbib}
\documentclass{article}

\usepackage[main, final]{neurips_2026}

\usepackage[utf8]{inputenc} 
\usepackage[T1]{fontenc}    
\usepackage{hyperref}       
\usepackage{url}            
\usepackage{booktabs}       
\usepackage{amsfonts}       
\usepackage{nicefrac}       
\usepackage{microtype}      
\usepackage{xcolor}

\usepackage{tikz}
\usetikzlibrary{fit,shapes.geometric} 
\usetikzlibrary{automata, positioning, arrows.meta}
\usepackage{nicematrix}
\usepackage{graphicx}
\usepackage{multirow}
\usepackage{amsmath,amssymb,amsfonts}
\usepackage{amsthm}
\usepackage{mathrsfs}
\usepackage{xcolor}
\usepackage{textcomp}
\usepackage{manyfoot}
\usepackage{booktabs}
\usepackage[ruled,vlined,linesnumbered]{algorithm2e}
\usepackage{algpseudocode}
\usepackage{listings}
\usepackage{bbm}
\usepackage{enumitem}
\usepackage{caption}
\usepackage{tcolorbox}
\newcommand{\bulletbullet}{\begin{array}{c}
\bullet\\
\bullet
\end{array}}
\newcommand{\bulletcirc}{\begin{array}{c}
\bullet\\
\circ
\end{array}}
\newcommand{\circcirc}{\begin{array}{c}
\circ\\
\circ
\end{array}}
\newcommand{\circbullet}{\begin{array}{c}
\circ\\
\bullet
\end{array}}
\usepackage{hyperref}

\newcommand{\indic}[1]{\mathbbm{1}_{\left\{#1\right\}}}
\newcommand{\indicS}[1]{\mathbbm{1}_{{#1}}}

\newcommand{\Prob}[1]{\mathbb{P}\left(#1\right)}
\newcommand{\Expt}[1]{\mathbb{E}\left(#1\right)}
\newcommand{\ExptWR}[2]{\mathbb{E}_{#1}\left(#2\right)}
\newcommand{\Var}[1]{\mathrm{Var}\left(#1\right)}
\newcommand{\VarWR}[2]{\mathrm{Var}_{#1}\left(#2\right)}
\newcommand{\CovWR}[3]{\mathrm{Cov}_{#1}\left(#2, #3\right)}

\newcommand{\aSet}[1]{\mathcal{#1}}

\newcommand{\Card}[1]{\left| #1 \right|}

\newcommand{\aij}[2]{a_{#1#2}}
\newcommand{\Jij}[2]{J_{#1#2}}

\newtheorem{definition}{Definition}
\newtheorem{theorem}{Theorem}
\newtheorem{lemma}{Lemma}
\newtheorem{corollary}{Corollary}
\newtheorem{proposition}{Proposition}

\newcommand{\locvecO}{\widetilde{O}}
\newcommand{\locvecx}{\widetilde{\times}}
\newcommand{\locvecm}{\widetilde{-}}

\newcommand{\MHEst}{\theta^{(\mathrm{MH})}}
\newcommand{\CMHEst}{\theta^{(\mathrm{CMH})}}

\newcommand{\nck}[2]{\left(\begin{array}{c}
#1 \\
#2
\end{array}\right)}

\title{Understanding Circulant Permutation to Extend the CMinHash Estimator}

\author{Keegan Kang  \\
  Department of Mathematics and Statistics\\
  Bucknell University\\
  Lewisburg, PA 17837 \\
  \texttt{kk054@bucknell.edu} \\
   \And
  Avery Hood \\
  Department of Mathematics and Statistics\\
  Bucknell University\\
  Lewisburg, PA 17837 \\
  \texttt{AveryHood1013@gmail.com} \\
   \AND
  Benedict H.W. Wong \\
  Institute for Infocomm Research (I$^2$R)\\
  A*STAR\\
  Singapore, 138632 \\
  \texttt{Benedict\_Wong@a-star.edu.sg}
}

\begin{document}

\maketitle

\begin{abstract}
CMinHash promises a new direction for designing MinHash algorithms with reduced variance and storage space via circulant permutations. However, the original variance analysis is difficult to extend to other MinHash estimators. We present a new approach that gives explicit and easily comparable approximate variance expressions, and use it to analyze the Minner estimator. Numerical experiments show that our variance expression matches the empirical mean square error (MSE) of estimates. Our work thereby demonstrates how to apply circulant permutation to other MinHash estimators and compute their approximate variance.
\end{abstract}

\section{Introduction}\label{sec1}

Consider two unweighted sets $\aSet{S}_1, \aSet{S}_2 \subseteq  \Omega$, where $\Omega = \{1,2, \hdots, D\}$. Let $\Card{\aSet{S}_1} = f_1$, $\Card{\aSet{S}_2} = f_2$, and $a = \Card{\aSet{S}_1 \cap \aSet{S}_2}$, and $f = \Card{\aSet{S}_1 \cup \aSet{S}_2}$. A common measure of similarity between these two sets is the Jaccard similarity, given by $J = \frac{\Card{\aSet{S}_1\cap \aSet{S}_2}}{\Card{\aSet{S}_1\cup \aSet{S}_2}} = \frac{a}{f_1+f_2-a}$. As $D$ is usually large, directly computing the pairwise similarity between these sets is not computationally feasible, and the task is to instead estimate the Jaccard similarity of these sets via hashing algorithms. Measuring this similarity is useful in many applications, even in non statistical (or computer science) disciplines \citep{ondov2016mash,rahman2025estimating,tseng2021parallel}. The traditional MinHash algorithm \citep{broder1997resemblance} uses random permutations to estimate $J$ between any two sets, and since then there has been a large body of work on estimating $J$, which we divide into three categories. 

{\bf New MinHash Approaches:} These are examples of MinHash algorithms that describe a novel (at the time) approach to MinHash. For example, {\it k-min sketches / Conditional Random Sampling} records more information with one random permutation  \citep{cohen2007summarizing,Church06conditionalrandom}, {\it b-bit minwise hashing} which stores the last $b$ bits of information \citep{li2010b, li2011theory}, {\it MinMaxHash} uses the ``other end" of the random permutation \citep{ji2013min}, {\it One Permutation Hashing} uses just one permutation \citep{li2012one,shrivastava2017optimal,shrivastava2014densifying}, the {\it Minner estimator} and records more data for asymmetric hashing \citep{ahleminner}. These contributions usually involve some kind of trade-off (e.g. data is required to be sparse for a better estimate, or greater computational time for reduced storage space), i.e. these methods do not work uniformly well for {\it all} datasets.

{\bf ``Framework" that complements existing MinHash methods:} These are techniques that can be applied to an existing algorithm to always improve its performance. The technique of {\it control variates and maximum likelihood estimation} can be applied to MinHash (and other similar algorithms) to always reduce their variance \citep{kang2021improving,li2011accurate}, while {\it CMinHash} shows that ``circulantly" moving rightwards across a permutation always reduces variance of estimates \citep{li2021cc,li2021c}.

{\bf Theoretical work which gives insight on the structure of MinHash:} For example, \cite{ahle2020subsets} gives a concise theorem on query time, space usage, and update time of similarity search algorithms, \cite{li2022minhash} explains how to construct a MinHash type estimator across continuous probability distributions.

Our work belongs to the last two categories, where we re-examine the approach taken in the CMinHash paper \citep{li2021c}, and propose a way to extend circulant permutations to other MinHash estimators. 

\begin{figure*}[h]
\begin{center}
\begin{align*}
    \begin{array}{c | c c c c c c c c c c c c c} 
    \text{\bf Permutation} & 6 & 1 & 7 & 4 & 5 & 8 & 3  & 2 \\ \hline
      \text{Set 1} & \bullet & \circ & \circ & \bullet &  \circ & {\bullet} & \circ & \boxed{\bullet} \\
      \text{Set 2} & \circ & \boxed{\bullet} & \bullet & \bullet  & \circ & \circ & {\bullet}  & \bullet\\
    \end{array}~~
    \begin{array}{c | c c c c c c c c c c c c c} 
    \text{\bf Draw Order} & 1 & 2 & 3 & 4 & 5 & 6 & 7 & 8  \\ \hline
      \text{Set 1} & \circ & \boxed{\bullet} & \circ & \bullet & \circ & \bullet & \circ & \bullet  \\
      \text{Set 2} &  \boxed{\bullet} & \bullet & \bullet & \bullet & \circ & \circ & \bullet & \circ \\
    \end{array}
\end{align*}
\end{center}
\caption{We have $\Card{\aSet{S}_1} = 4$, $\Card{\aSet{S}_2} = 5$. By treating a {\it random permutation} (left) as the indices where we {\it sample without replacement} (right), $\min\{\pi(\aSet{S}_1)\} = 2$, $\min\{\pi(\aSet{S}_2)\} = 1$, which is just the index where the first $\bullet$ appears.}\label{fig1eg}
\end{figure*}

\subsection{Related Work}

Without loss of generality, let $\aSet{S}_1, \aSet{S}_2 \subseteq  \Omega$, and we want to estimate the {\it Jaccard similarity} $J$ between these sets. These sets can be represented as $D$ sequential objects consisting of one of four types: $\left\{\bulletbullet, \bulletcirc, \circbullet, \circcirc   \right\}$, where the $\bullet$ in the corresponding row denotes an element of that set. 
\begin{theorem}\label{OGMH}
(Broder's MinHash) Suppose we have $K$ independent random permutations $\pi_1, \hdots, \pi_K$ on $[D]$. Let $\MHEst = \sum_{t=1}^K\indic{\min\{\pi_t(\aSet{S}_1)\} = \min\{\pi_t(\aSet{S}_2)\}}/K$. Then $\Expt{\MHEst} = J$, with $\Var{\MHEst} = J(1-J)/K$.
\end{theorem}

A {\it random permutation} $\pi$ on $[D]$ can be thought of as a sequence of draws under {\it sampling without replacement}, where $\min\{\pi(\aSet{S}_1)\}$ is the draw where the first $\bullet$ appears in the first row, and $\min\{\pi(\aSet{S}_2)\}$  as the draw where the first $\bullet$ appears in the second row.  Figure~\ref{fig1eg} shows an example where $D = 8$. Let the set $\aSet{F}_{+} = \left\{\bulletbullet\right\}$, and $\aSet{F}_{-} = \left\{\circbullet, \bulletcirc \right\}$. Each permutation in Broder's MinHash is equivalent to a trial where a ``success" is drawing an element in $\aSet{F}_{+}$ before an element in $\aSet{F}_{-}$, and an estimator for $J$ is found by finding the probability that an element from  $\aSet{F}_{+}$ {\bf is seen before an element in} $\aSet{F}_{-}$. Hence by Theorem~\ref{OGMH}, we get an estimate for $J$ by recording the proportion of ``successes" for $K$ independent random permutations (trials), $1 \leq t \leq K$.

More generally, other MinHash algorithms can be represented by different $D$ sequential objects, $\aSet{F}_{+}$ and $\aSet{F}_{-}$, and an estimator for $J$ can be found by finding the probability that $n_1$ elements from $\aSet{F}_{+}$ are seen before $n_2$ elements in $\aSet{F}_{-}$ (we give an example in Section~\ref{sec4}).

\begin{figure*}[h]
\begin{center}
\begin{align*}
    \begin{array}{c | c c c c c c c c c c c c c} 
\text{\bf Permutation} & 3 & 2 & 6 & 1 & 7 & 4 & 5 & 8  \\ \hline
      \text{Set 1} & \bullet & \circ & \circ & \boxed{\bullet} &  \circ & {\bullet} & \circ & {\bullet} \\
      \text{Set 2} & \circ & {\bullet} & \bullet & \boxed{\bullet}  & \circ & \circ & {\bullet}  & \bullet\\
    \end{array}~~~~~
    \begin{array}{c | c c c c c c c c c c c c c} 
    \text{\bf Draw Order} & 1 & 2 & 3 & 4 & 5 & 6 & 7 & 8  \\ \hline
      \text{Set 1} & \boxed{\bullet} & \circ & \bullet & \bullet & \circ & \circ & \circ & \bullet  \\
      \text{Set 2} & \boxed{\bullet} & \bullet & \circ & \circ & \bullet & \bullet & \circ & \bullet \\
    \end{array}
\end{align*}
\end{center}
\caption{The permutation $\pi$ in Figure~\ref{fig2eg} is shifted rightward from Figure~\ref{fig1eg} by 2 units.  Then $\min\{\pi_{\rightarrow 2}(\aSet{S}_1)\} = \min\{\pi_{\rightarrow 2}(\aSet{S}_2)\} = 1$.}\label{fig2eg}
\end{figure*}

CMinHash extends Broder's MinHash by reducing the $K$ independent permutations needed to two independent permutations, given by $\sigma, \pi$, and we write this as CMinHash$(\sigma,\pi)$ following the notation in \cite{li2021c}. First, the indices in the sets $\aSet{S}_i$ are permuted  by the random permutation $\sigma$. Second, the random permutation $\pi$ is generated, and is re-used $K$ times by shifting $\pi$ circulantly rightwards by 1 unit $K$ times. More formally, let $\pi_{\rightarrow t}$ denote shifting the permutation $t$ units to the right. Figure~\ref{fig2eg} gives an example of shifting the permutation in Figure~\ref{fig1eg} rightwards by 2 units. By linearity of expectations, $\ExptWR{\sigma,\pi}{\sum_{t=1}^K\indic{\min\{\pi_{\rightarrow t}(\aSet{S}_1)\} = \min\{\pi_{\rightarrow t}(\aSet{S}_2)\}}/K} = J$, hence recording the proportion of ``successes" under $K$ shifts gives an estimate of $J$. A less compact version is on Page~\pageref{CMHAlgApp} in Appendix~\ref{appendixA1ALGO}. However, the variance of this estimate is slightly complicated, although lower than Broder's MinHash (see Page~\pageref{lithm31long} in Appendix~\ref{appendixA2SPAM} for  expression). 

\subsection{Our Contributions And Paper Organization}

CMinHash($\sigma, \pi$) promises a new direction to design hashing algorithms using the 2 permutations ($\sigma, \pi$) and $K$ circulant shifts of $\pi$ (saving on storage space), but having a smaller variance than $K$ independent permutations under traditional MinHash (and similar algorithms). 

While the algorithm behind CMinHash($\sigma, \pi$) is straightforward to implement (see Algorithm~\ref{CMHAlgApp} in Appendix~\ref{appendixA1ALGO}), the theoretical analysis of the variance forms a substantial bulk of the original paper. 

Moreover, the approach the original paper adopts in terms of random permutations makes the proofs of their Lemma 2.4 (used to compute expectations under $\sigma$) and Theorem 3.1 (used to compute the variance of the estimates under $\sigma$ and $\pi$) difficult to replicate in other settings due to the random permutation paradigm and combinatorial calculations. Moreover, the result of Theorem 3.1 (variance of the CMinHash estimator) in \cite{li2021c} gives the variance as a combinatorial sum under constraints. Appendix~\ref{appendixA2SPAM} notes that one equality used in their proof does not hold in general. This makes it difficult to theoretically compare the CMinHash variance against variance of other algorithms, since Theorem 3.4 and Proposition 3.6 show that the variance (given by the combinatorial sum) is only lower than Broder's MinHash estimate, and says nothing about other MinHash algorithms. The full statements of relevant lemmas, theorems, and propositions in \cite{li2021c} can be found in Appendix~\ref{appendixA2SPAM}.

Hence, researchers using a circulant permutation to improve other MinHash estimators would need to independently show that using a circulant permutation theoretically gives a lower variance compared to different MinHash algorithms, {\it which is non-trivial if following similar approaches as in the original paper}. Secondly, researchers may also want to compare the performance of two algorithms both using a  circulant permutation framework. This is difficult to do theoretically, as the circulant permutation variance is given as a combinatorial sum under constraints if using the approach in \cite{li2021c}.

While simulations can be run to compare the empirical performances of a proposed circulant permutation estimator against baselines and sidestep theory, there is no guarantee that a proposed estimator always has reduced variance in all simulated cases, since estimators can have different performance when datasets vary (e.g. sparsity of dataset).  This leads us to give the following contributions.

{\bf Contribution One:} We give a new proof of Lemma 2.4 in \cite{li2021c}, using a different paradigm in Section~\ref{sec2}. This allows {\bf simpler computation of expectations under the permutation $\sigma$}. This approach can be applied to other MinHash estimators, and we give an example in Section~\ref{sec4}.

{\bf Contribution Two:} We derive a simpler expression for the variance in Theorem 3.1 in \cite{li2021c} under the assumption that $D \rightarrow \infty$, which is common  (e.g. see \cite{li2011theory,li2012one,shrivastava2014densifying}). Our result does not require the equality discussed in Appendix~\ref{appendixA2SPAM}, and is slightly stronger in this respect. This allows {\bf easier analysis and comparison of the variance under the permutations $\sigma$ and $\pi$}. More specifically, this gives a {\bf simpler variance expression which is easier to analyze (and compare)}
with a remainder of order $f/D^2$ when $f/D\rightarrow0$, and of order $D^{-2}$ when $f/D$ is bounded away from zero, which we explain in Section~\ref{sec2}. We remark that this gives approximately accurate variances even for small $D = 128$. This approach can also be applied to other MinHash estimators to get variance expressions for the corresponding circulant permutation estimator, with an example in Section~\ref{sec4}.

{\bf Contribution Three:} We show how our approach can be used to {\bf analyze MinHash estimators based on circulant permutation}, which we demonstrate on an asymmetric hashing setup with the Minner estimator. We pick the Minner estimator for two reasons: a) to show how to do variance analysis when the sets $\mathcal{F}_+$ and $\mathcal{F}_-$ are different, and b) how to extend the circulant permutation setup with counts, i.e.  finding the expected number of elements seen in  $\aSet{F}_{-}$ before an element in $\aSet{F}_{+}$.

Overall, our paper {\bf provides a direct way to compute expectations and variance under circulant permutation, making it easier to analyze and extend MinHash estimators.}

Section~\ref{sec2} (and \ref{appendixA1ALGO}-\ref{appendixA2SPAM}) give a broader overview of CMinHash. Section~\ref{sec4} applies our approach to the Minner estimator involving counts of events. Section~\ref{sec5} discusses the importance of our results.

\section{CMinHash in Detail}\label{sec2}

The sets $\aSet{S}_1, \aSet{S}_2 \subseteq \{1,2,\hdots,D\}$ described on Page \pageref{sec1} can be visualized as $D$ objects in sequential order (see Figure~\ref{fig1eg}). Let $a = \#$ of $\bulletbullet$, $f-a$ be the $\#$ of $\circbullet$ and $\bulletcirc$, and $D-f$ be the $\#$ of $\circcirc$. 

\begin{definition}
Definition 2.1 in \cite{li2021c} (modified). Define the location vector $\vec{x}$, where  $x_i \in \{\locvecO, \locvecx, \locvecm\}$. $x_i$ corresponds to the type of the $i^{\text{th}}$ object, which is a $\locvecO$ if $\bulletbullet$, $\locvecx$ if $\circbullet$ or $\bulletcirc$, and $\locvecm$ if $\circcirc$.
\end{definition}

Let $\CMHEst := \sum_{t=1}^K\indic{\min\{\pi_{\rightarrow t}(\aSet{S}_1)\} = \min\{\pi_{\rightarrow t}(\aSet{S}_2)\}}/K$ which is the CMinHash estimator. We write $\indicS{t} \equiv \indic{\min\{\pi_{\rightarrow t}(\aSet{S}_1)\} = \min\{\pi_{\rightarrow t}(\aSet{S}_2)\}}$ for compactness. $\Expt{\CMHEst}$ is $J$ due to linearity of expectations (equal probability of $J$ regardless of starting point after applying $\pi$), but computing $\Var{\CMHEst}$ requires knowing $\CovWR{\sigma,\pi}{\indicS{s}}{\indicS{t}}$, $1 \leq s, t \leq K$, as $\indicS{s}, \indicS{t}$ are correlated. 

\begin{definition}\label{defn22}
Definition 2.2 in \cite{li2021c}. For $A,B \in \{\locvecO, \locvecx, \locvecm\}$, let $\{(A,B)~|~\Delta\}$ be the set $\{(i,j): (x_i,x_j) = (A,B), j-i = \Delta\}$. For $1 \leq \Delta \leq K-1$, define
\begin{align*}
\hspace*{-0.2cm}
\begin{array}{c c c}
\mathcal{L}_0(\Delta) = \{(\locvecO,\locvecO)|\Delta\}, & \hspace*{-0.1cm} \mathcal{L}_1(\Delta) = \{(\locvecO,\locvecx)|\Delta\}, &  \hspace*{-0.1cm} \mathcal{L}_2(\Delta) = \{(\locvecO,\locvecm)|\Delta\}, \\
\mathcal{G}_0(\Delta) = \{(\locvecm,\locvecO)|\Delta\}, &  \hspace*{-0.1cm} \mathcal{G}_1(\Delta) = \{(\locvecm,\locvecx)|\Delta\}, &  \hspace*{-0.1cm} \mathcal{G}_2(\Delta) = \{(\locvecm,\locvecm)|\Delta\}, \\
\mathcal{H}_0(\Delta) = \{(\locvecx,\locvecO)|\Delta\}, &  \hspace*{-0.1cm} \mathcal{H}_1(\Delta) = \{(\locvecx,\locvecx)|\Delta\}, &  \hspace*{-0.1cm} \mathcal{H}_2(\Delta) = \{(\locvecx,\locvecm)|\Delta\},
\end{array}
\end{align*}
with the following constraints for all $1 \leq \Delta \leq K-1$,
\begin{align*}
|\mathcal{L}_0| + |\mathcal{L}_1| + |\mathcal{L}_2| = |\mathcal{L}_0| + |\mathcal{G}_0| + |\mathcal{H}_0| &= a,  \\
|\mathcal{G}_0| + |\mathcal{G}_1| + |\mathcal{G}_2| = |\mathcal{L}_2| + |\mathcal{G}_2| + |\mathcal{H}_2| &= D - f, \\
|\mathcal{H}_0| + |\mathcal{H}_1| + |\mathcal{H}_2| = |\mathcal{L}_1| + |\mathcal{G}_1| + |\mathcal{H}_1| &= f-a. 
\end{align*}
\end{definition}

The nine tuples $(A,B)$ in Definition~\ref{defn22} correspond to symbols at $x_i, x_{i+\Delta}$ in the location vector $\vec{x}$.

\newcommand{\lilemma}{Lemma 2.4 in \cite{li2021c}. For any $1 \leq s < t \leq K$, with $t - s = \Delta$, then $\mathbb{E}_\pi\left(\indicS{s}\indicS{s+\Delta}\right) = \frac{\Card{\mathcal{L}_0(\Delta)} + (\Card{\mathcal{G}_0(\Delta)} + \Card{\mathcal{L}_2(\Delta)}    )J  }{f + \Card{\mathcal{G}_0(\Delta)} + \Card{\mathcal{G}_1(\Delta)}}$.}
\begin{lemma}\label{lemma24}
\lilemma
\end{lemma}
We give a simpler proof in Appendix~\ref{appendixa3_SIMPLE24} with a pictorial intuition in Table~\ref{proof_by_picture}, and explain where our approach differs to \cite{li2021c}. Following the sampling without replacement paradigm where a random permutation can be thought of as the {\it draw order}, then for a fixed $\Delta$, then applying $\pi$ means we are drawing the indices $(x_s, x_{s+\Delta})$ from the location vector $\vec{x}$ in some random order, and we know how many of each $(A,B) = (x_s, x_{s+\Delta})$ indices there are using Definition~\ref{defn22}. The first entry in the tuple $(A,B)$ then corresponds to what is seen when $\pi$ is shifted $s$ times to the right, and the second entry corresponds to what is seen when $\pi$ is shifted $s + \Delta$ times to the right. Since a $\locvecO$ in the location vector means that the hashed values collide, then $\ExptWR{\pi}{\indicS{s}\indicS{s+\Delta}} > 0$ can be thought of as the {\bf probability of a sequence of draws} of the nine different types of $D$ tuples where the first $\locvecO$ appears before the first $\locvecx$ for both entries of the tuple, as the random permutation $\pi$ is a sequence of draws without replacement. We call such a sequence of draws {\bf a success}, with two examples below.

\begin{itemize}
\item $(\locvecm,\locvecm), (\locvecm,\mathbf{\locvecO}), (\mathbf{\locvecO},\locvecm)$  
\item $(\locvecm,\locvecm), (\locvecm,\mathbf{\locvecO}), (\locvecm, \locvecx), (\mathbf{\locvecO},\locvecO)$. The first $\locvecO$ for the second entry of the tuples appeared before the first $\locvecx$ for the second entry of the tuples.
\end{itemize} 
More examples are given by Table~\ref{examplethingya} and Table~\ref{examplethingyb} in Appendix~\ref{appendixa3_SIMPLE24}. Hence, we just need to compute the probabilities of seeing the tuple we want before other invalid tuples, which is much simpler. The intuition behind this simplification can be seen behind this following analogy: Consider a bag with one blue marble, two red marbles, and $N$ green marbles. The probability of drawing a blue marble before a red marble is always $\frac{1}{3}$, regardless of $N$. We can do the same with $\locvecO$ and $\locvecx$, regardless of $\locvecm$.

Next, we compute the variance of the CMinHash$(\sigma,\pi)$ estimator under $K$ shifts adopting a different approach, which we outline here. Suppose $D$ is large. Under the random permutation $\sigma$, both elements in the tuple $(x_s,x_{s+\Delta})$ can be thought of drawing entries $x_s, x_{s+\Delta}$ without replacement from a population of size $D$ with three different objects $\locvecO, \locvecx, \locvecm$. More specifically, each set $\{(A,B)~|~\Delta\}, 1 \leq \Delta \leq K-1$ can be expressed as
\begin{align}
\Card{|\{(A,B)~|~\Delta\}} = \sum_{s=1}^D \indic{(x_s,x_{s+\Delta}) = (A,B)}. \label{expec_set}
\end{align}
Therefore, we can take expectations across (functions of) \eqref{expec_set}. For example, 
\begin{align*}
\ExptWR{\sigma,\pi}{|\mathcal{L}_0(\Delta)|} & = \sum_{s=1}^D \ExptWR{\sigma,\pi}{\indic{(x_s,x_{s+\Delta}) = (\locvecO,\locvecO)}} = \frac{a(a-1)}{D-1}.  
\end{align*}
For the sparse case, we first condition on the positions of the $f$ entries which are not $\locvecm$. This gives an exact expression whose first term can be computed by considering when two such entries are next to each other. For the non sparse case, we use the delta method on the same exact expression. The details are in Appendix~\ref{appendix_a4_LI}.

\begin{lemma}\label{taylorLemma}
Suppose we pre-process the location vector with the random permutation $\sigma$. Let $B_D = 2D-f-1$, and suppose $0 < J <1$. For all $1 \leq \Delta \leq K-1$, if $f/D \rightarrow 0$, then
\begin{align*}
\ExptWR{\sigma,\pi}{\indicS{s}\indicS{s+\Delta}} & = \left\{\begin{array}{r l}
J^2 - \dfrac{J(1-J)f}{2(f-1)(D-1)}+O\left(\dfrac{f}{D^2}\right), & D\text{ even and }\Delta=D/2,\\
J^2-\dfrac{J(1-J)f}{(2f-1)(D-1)}+O\left(\dfrac{f}{D^2}\right), & \text{otherwise},
\end{array}\right. \\
\CovWR{\sigma,\pi}{\indicS{s}}{\indicS{s+\Delta}}
 & = \left\{\begin{array}{r l}
-\dfrac{J(1-J)f}{2(f-1)(D-1)} + O\left(\dfrac{f}{D^2}\right), & D\text{ even and }\Delta=D/2,\\
-\dfrac{J(1-J)f}{(2f-1) (D-1)} +O\left(\dfrac{f}{D^2}\right), & \text{otherwise}.
\end{array}\right.
\end{align*}
If $f/D$ is bounded away from zero, then
\begin{align*}
\ExptWR{\sigma,\pi}{\indicS{s}\indicS{s+\Delta}}
&=J^2-\frac{J(1-J)}{B_D}+O\left(D^{-2}\right),\\
\CovWR{\sigma,\pi}{\indicS{s}}{\indicS{s+\Delta}}
&=-\frac{J(1-J)}{B_D}+O\left(D^{-2}\right).
\end{align*}
\end{lemma}

We prove this in Appendix~\ref{appendix_a4_LI}.

\begin{definition}\label{lambda_defn}
For $0 < K \leq D$, let $\lambda_{D,K} = \indic{D~ \text{even}} \frac{(2K-D)_{+}}{K}$, where $(x)_{+} : = \max\{x,0\}$.
\end{definition}
Definition~\ref{lambda_defn} accounts for the proportion of terms in the covariance where $\Delta = D/2$.

\begin{theorem}\label{ourownMAINtheorem}
Let $0 < a < f \leq D$. Suppose we pre-process the location vector with the random permutation $\sigma$, and generate an independent permutation $\pi$. Let $B_D = 2D-f-1$. If $f/D \rightarrow 0$, 
\begin{align*}
\VarWR{\sigma,\pi}{\CMHEst}
=\frac{J(1-J)}{K}\left[1-\frac{f}{D-1}\left\{\frac{K-1-\lambda_{D,K}}{2f-1}+\frac{\lambda_{D,K}}{2(f-1)}\right\}\right]
+O\left(\frac{f}{D^2}\right).
\end{align*}
If $f/D$ is bounded away from zero, then
\begin{align*}
\VarWR{\sigma,\pi}{\CMHEst}
=\frac{J(1-J)}{K}\left(1-\frac{K-1}{B_D}\right)+O\left(D^{-2}\right).
\end{align*}
\end{theorem}
\begin{proof}
We give the proof in Appendix~\ref{appendix_a4_LI}.
\end{proof}

The proof of Lemma~\ref{taylorLemma} also gives an exact negative expression for every covariance. Hence, for $K>1$, the variance under CMinHash$(\sigma,\pi)$ is smaller than the variance $J(1-J)/K$ under $K$ independent permutations. If $f/D$ is bounded away from zero, the first-order variance reduction can be thought of as approximately scaled by $\left(1-(K-1)/(2D-f-1)\right)$. The leading variance ratio relative to MinHash is independent of $a$ when $D,f,K$ are fixed, which also verifies the Uniform Superiority (Theorem 3.4) and Consistent Improvement (Proposition 3.6) statements in \cite{li2021c}.

\begin{figure*}[h]
\includegraphics[width=1.0\linewidth]{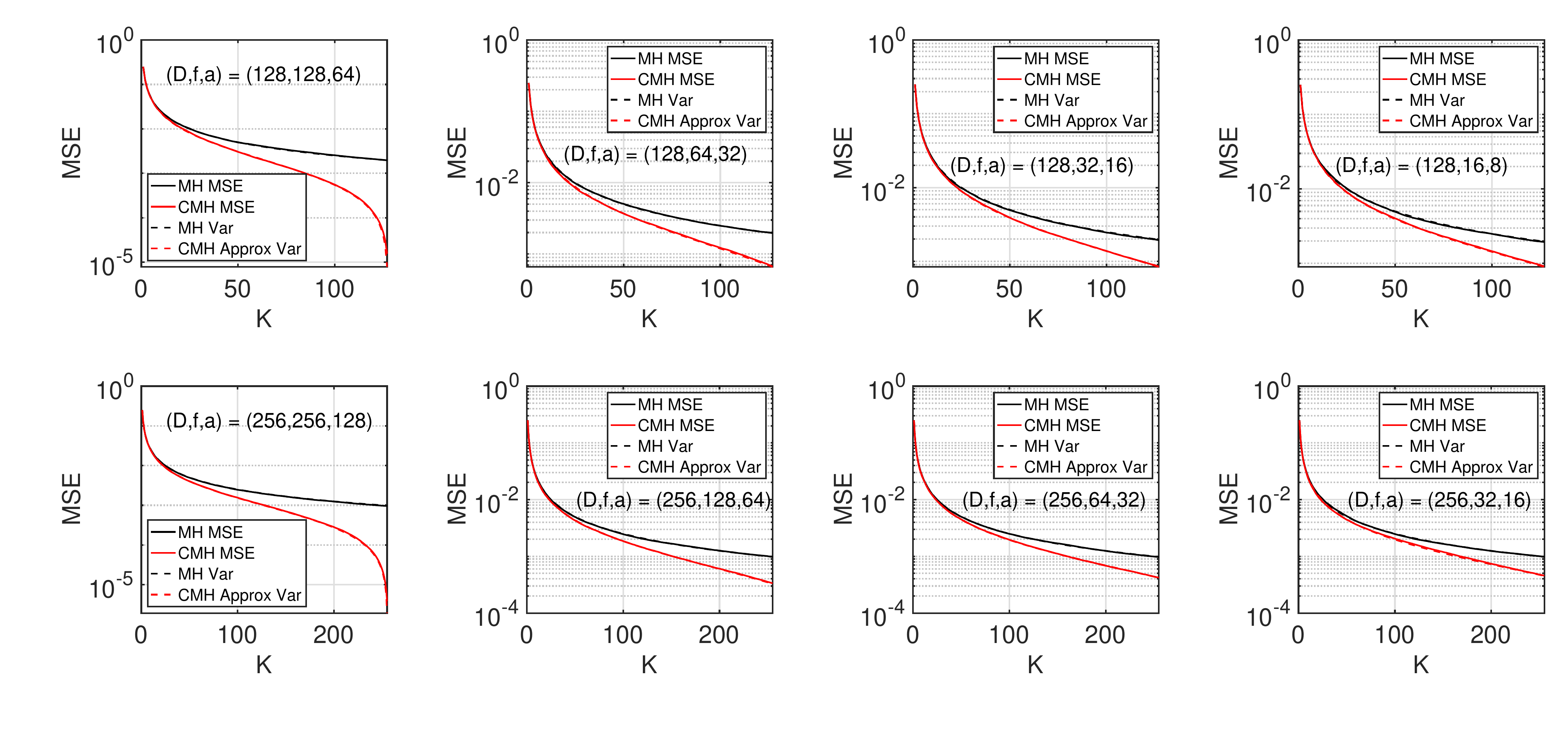}
\caption{Empirical MSE vs theoretical approximate variance in Theorem~\ref{ourownMAINtheorem} of MinHash and CMinHash $(\sigma,\pi)$. The dotted lines overlap almost perfectly (up to Monte Carlo error), which shows our theoretical approximate variance matches the empirical variance. \label{fig1CMH}}
\end{figure*}

We first construct a similar experiment to \cite{li2021c}, generating synthetic vectors with varying $D,f,a$ in Figure~\ref{fig1CMH}, averaging over 10000 iterations, and overlaying the  theoretical approximation with the empirical mean square error (MSE) of the estimates. The dotted lines overlap almost perfectly with the solid lines at $D=128$ and $D=256$.

We do more experiments in Appendix~\ref{appendix_a5_EXPT} (Figure~\ref{fig2CMH}) to investigate if the approximation is accurate in edge cases when $f/D\rightarrow 0$. At $D=128$, for $f=2,4$, the variance approximation is slightly higher than the empirical MSE. But at $D=1024$ for $f=2,4$, both lines overlap each other almost perfectly. We do more experiments on the Words dataset \citep{li2005using} (Figure~\ref{words_data_CMH} in Appendix~\ref{appendix_a5_EXPT}), where the theoretical variance approximation and empirical MSE lines also overlap each other almost perfectly.

Hence, we believe our approximation of the theoretical variance is practical enough to do comparisons.

We argue that our approach in Lemma~\ref{taylorLemma} should be how MinHash algorithms are analyzed, especially with large $D$, {\it as large $D$ necessitated the need of MinHash in the first place}.

\section{Applying the Circulant Permutation Framework To Asymmetric Hashing} \label{sec4}

A natural question to ask is if our work is generalizable to other MinHash algorithms. We show our approach can be applied to the Minner estimators \citep{ahleminner} (asymmetric hashing), and theoretically demonstrate a circulant permutation framework with reduced variance.

The asymmetric hashing setup considers a baseline set $\aSet{S}_1$ relative to $\aSet{S}_2, \hdots, \aSet{S}_n$, and estimates the set overlap $\aij{1}{i}$ (which can then be used to estimate the Jaccard similarity $\Jij{1}{i}$, since $\aij{1}{i} = \frac{(f_1+f_i)\Jij{1}{i}}{1 + \Jij{1}{i}}$), under the assumption that the set cardinalities $f_1, \hdots, f_n$ are all known. We give more details of the Minner estimators in Appendix~\ref{appendixB_understandingMinner}.  

\begin{figure}[h]
\begin{center}
\begin{align*}
    \begin{array}{c | c c c c c c c c c c c c c} 
    \text{\bf Draw Order} & 1 & 2 & 3 & 4 & 5 & 6 & 7 & 8  \\ \hline
      \text{Set 1} & \circ & \bullet & \circ & \bullet & \bullet & \bullet & \circ & \bullet  \\
      \text{Set 2} &  \circ & \circ & \circ & \circ & \circ & \circ & \circ & \boxed{\bullet} \\
    \end{array}
\end{align*}
\end{center}
\caption{We have $\min\{\pi_t(\aSet{S}_2)\} = 8$, with $y_t = 1$, which can be thought of as the presence (or absence) of a $\bullet$ at draw 8 for $\aSet{S}_1$, and $m_t = 4$, which can be thought of as counting how many $\bullet$ were seen in previous draws 1-7 for $\aSet{S}_1$.}\label{fig4eg}
\end{figure}

\begin{theorem}
Minner estimators (heavily modified from \citep{ahleminner})\label{minner_haha}
Suppose we have $K$ independent random permutations $\pi_1,\hdots,\pi_K$ on $[D]$, and we want to estimate the set overlap $a$ between $\aSet{S}_1$ and $\aSet{S}_2$, where $\aSet{S}_1$ is the baseline set. Let $y_t = \indic{ \min(\pi_t(\aSet{S}_2)) \in\pi_t(\aSet{S}_1) }$, and let $m_t$ be the count of elements in $\pi_t(\aSet{S}_1)$ strictly less than $\min(\pi_t(\aSet{S}_2))$. Then the two estimators $f_2\sum_{t=1}^K\frac{y_t}{K}$ and $f_1-(f_2+1)\sum_{t=1}^K\frac{m_t}{K}$ both have expectation $a$. Their variances are, respectively, $\frac{a(f_2-a)}{K}$ and $\frac{f_2(f_1-a)(f_1+f_2-a+1)}{K(f_2+2)}$.
\end{theorem}

We remark that thinking of the random permutation as the draw order in sampling without replacement (see Figure~\ref{fig4eg}) makes understanding the Minner estimators straightforward.

\subsection{Applying the Circulant Permutation framework to a different setup of the location vector}

Without loss of generality, consider estimating the set overlap between $\aSet{S}_1$ and $\aSet{S}_2$, where we look at the random variable $y_t$. The question we aim to resolve is: {\it How can the circulant permutation framework be applied when we are not recording both $\min\{\pi_{\rightarrow t}(\aSet{S}_i)\}$}?

In Section~\ref{sec2}, we introduced $\aSet{F}_+$ and  $\aSet{F}_{-}$, where we find the probability that an element from $\aSet{F}_{+}$ is seen before an element from $\aSet{F}_{-}$. In the context of $y_t$, we set $\aSet{F}_{+} = \left\{\bulletbullet\right\}$, and $\aSet{F}_{-} = \left\{\circbullet\right\}$, as we continue drawing if we see a $\bulletcirc$ or $\circcirc$. Let $C_{s} = \#$ of elements of $\aSet{F}_+$ in $D$ (success), $C_f = \#$ of elements of $\aSet{F}_-$ in $D$ (failure), and $C_n = D-C_s-C_f$ (neutral). There are $C_s = a$ $\locvecO$s, $C_f = f_2-a$ $\locvecx$s, and $C_n = D-f_2$ $\locvecm$s. To generalize this, the constraints of the sets are now 
\begin{align*}
|\mathcal{L}_0| + |\mathcal{L}_1| + |\mathcal{L}_2| = |\mathcal{L}_0| + |\mathcal{G}_0| + |\mathcal{H}_0| &= C_s,  \\
|\mathcal{G}_0| + |\mathcal{G}_1| + |\mathcal{G}_2| = |\mathcal{L}_2| + |\mathcal{G}_2| + |\mathcal{H}_2| &= C_n, \\
|\mathcal{H}_0| + |\mathcal{H}_1| + |\mathcal{H}_2| = |\mathcal{L}_1| + |\mathcal{G}_1| + |\mathcal{H}_1| &= C_f.
\end{align*}
It is clear that under one permutation, the variance of drawing an element from $\aSet{F}_{+}$ before $\aSet{F}_{-}$ is $\frac{C_sC_f}{(C_s+C_f)^2}$ regardless of the size of $C_n$. Let $V_{\mathrm{ind}}$ be the variance of 1 trial, and $V_{K}$ the variance of $K$ independent trials (i.e. under $K$ independent permutations). This leads to the following corollary.
\begin{corollary}\label{CMHCorollaryGen}
Suppose we pre-process the location vector with a random permutation $\sigma$, and let $\indicS{s} \equiv \indic{\text{Element in $\aSet{F}_{+}$ is seen before element in $\aSet{F}_{-}$ at $\pi_{\rightarrow s}$}}$.
Then for any $1 \leq s < t \leq K$, with $t - s = \Delta$,   
\begin{align*}
\ExptWR{\pi}{\indicS{s}\indicS{s+\Delta}}  = \frac{\Card{\mathcal{L}_0(\Delta)} + (\Card{\mathcal{G}_0(\Delta)} + \Card{\mathcal{L}_2(\Delta)}) \frac{C_s}{C_s+C_f}  }{C_s + C_f + \Card{\mathcal{G}_0(\Delta)} + \Card{\mathcal{G}_1(\Delta)}}.
\end{align*}
Let $\lambda_{D,K}$  be defined as in Definition~\ref{lambda_defn}, and let $B_D = D+C_n-1$. 
If  $(C_s+C_f)/D \rightarrow 0$, then 
{\small
\begin{align*}
\VarWR{\sigma,\pi}{\sum_{t=1}^K\frac{\indicS{t}}{K}}  
&=V_K\left[1-\frac{C_s+C_f}{D-1}\left\{
\frac{K-1-\lambda_{D,K}}{2(C_s+C_f)-1}+\frac{\lambda_{D,K}}{2(C_s+C_f-1)}\right\}\right]
+O\left(\frac{C_s+C_f}{D^2}\right).
\end{align*} 
}

If $(C_s+C_f)/D$ is bounded away from zero, then 
\begin{align*}
\VarWR{\sigma,\pi}{\sum_{t=1}^K\frac{\indicS{t}}{K}}  
=V_K\left(1-\frac{K-1}{B_D}\right)+O(D^{-2}).
\end{align*}
\end{corollary}

\begin{proof}
We substitute $a=C_s$ and $f=C_s+C_f$ in Lemma~\ref{lemma24}, Lemma~\ref{taylorLemma}, and Theorem~\ref{ourownMAINtheorem}. Since $D-f=C_n$, the quantity $2D-f-1$ becomes $D+C_n-1=B_D$.
\end{proof}

By the exact covariance identity \eqref{exactCMHcovariance2}, the variance under $K > 1$ circulant shifts is always reduced compared to $K$ independent permutations. The size of the reduction depends on $(C_s+C_f)$. In particular, MinHash algorithms where $\aSet{F}_{+}$ and $\aSet{F}_{-}$ are small relative to $D$ have a small variance reduction when CMinHash$(\sigma,\pi)$ is applied. 

With respect to the Minner estimator $y_t$, we can apply Corollary~\ref{CMHCorollaryGen} and immediately get the theoretical approximation of the variance of $y_t$ by letting $C_s = a$, $C_f = f_2-a$. 
\subsection{Applying the Circulant Permutation framework to counts}

Now consider estimating the set overlap between $\aSet{S}_1$ and $\aSet{S}_2$, where we look at the random variable $m_t$. The question we aim to resolve is: {\it How can the circulant permutation framework be applied when we are recording {\bf counts} of elements in $\mathcal{F}_{-}$ before the first element in $\mathcal{F}_{+}$ is seen?}

We still construct the two sets $\aSet{F}_{+}$, $\aSet{F}_{-}$, but give a different meaning for elements in $\aSet{F}_{-}$. Let $\aSet{F}_{+} = \left\{ \bulletbullet, \circbullet \right\}$, and $\aSet{F}_{-} = \left\{\bulletcirc \right\}$ with $C_s = f_2$ and $C_f = f_1-a$. Then the setup becomes: ``count how many elements in $\aSet{F}_{-}$ are seen before the first element in $\aSet{F}_{+}$". Equivalently this would be ``count how many $\locvecx$ are seen before the first $\locvecO$ is drawn" for the location vector $\vec{x}$.

From Lemma~\ref{emptylemma} in Appendix~\ref{appendixB_proofLemma4}, the variance of observing $\theta$ elements from $\aSet{F}_{-}$ before the first $\aSet{F}_{+}$ is $\frac{C_sC_f(C_s+C_f+1)}{(C_s+1)^2(C_s+2)}$. Let $V_{\mathrm{ind}}$ be the variance of 1 trial, and $V_{K}$ the variance of $K$ independent trials.

\newcommand{\lastlemmaA}{Suppose we pre-process the location vector with a random permutation $\sigma$. Let the random variable $\theta_s \equiv $ number of $\locvecx$ seen before the first $\locvecO$ under the random permutation $\pi_{\rightarrow s}$. Then for any $1 \leq s < t \leq K$, with $t - s = \Delta$,
}

\newcommand{\lastlemmaexpra}{\frac{\Card{\aSet{H}_1}}{2C_s - \Card{\aSet{L}_0} +1}}

\newcommand{\lastlemmaexprb}{\frac{2C_f^2 - C_f(\Card{\aSet{L}_1} + \Card{\aSet{H}_0}) - 2 \Card{\aSet{H}_1}}{(C_s+1)(2C_s - \Card{\aSet{L}_0} +2)} }

\newcommand{\lastlemmaexprc}{\frac{C_f(C_s^2 + 2C_sD - 6D^2) + C_s^2(C_s-2D)}{C_s^3(2D-C_s)^2}}

\newcommand{\lastlemmaexpr}{\begin{align*}
\ExptWR{\pi}{\theta_s \theta_{s+\Delta}} & = \lastlemmaexprb + \lastlemmaexpra. \\
\end{align*}}

\begin{lemma}
\label{lastlemA}
Suppose we pre-process the location vector with a random permutation $\sigma$. Let the random variable $\theta_s \equiv $ number of $\locvecx$ seen before the first $\locvecO$ under the random permutation $\pi_{\rightarrow s}$. Let $B_{\Delta} = 2C_s - \Card{\mathcal{L}_0} + 1$. Then for any $1 \leq s < t \leq K$, with $t - s = \Delta$,
\begin{align*}
\ExptWR{\pi}{\theta_s \theta_{s+\Delta}}
&=
\frac{\Card{\mathcal{H}_1}}{B_{\Delta}}
+
\frac{2C_f^2-C_f(\Card{\mathcal{L}_1}+\Card{\mathcal{H}_0})-2\Card{\mathcal{H}_1}}
{(C_s+1)(B_{\Delta}+1)} +
\frac{C_f(\Card{\mathcal{L}_1}+\Card{\mathcal{H}_0})-2\Card{\mathcal{L}_1}\Card{\mathcal{H}_0}}
{(C_s+1)B_{\Delta}(B_{\Delta}+1)}.
\end{align*}
If $C_s$ and $C_f$ are fixed as $D\rightarrow\infty$, then, when $D$ is even and $\Delta=D/2$,
\begin{align*}
\CovWR{\sigma,\pi}{\theta_s}{\theta_{s+\Delta}}
&=-V_{\mathrm{ind}}
\frac{(C_s+1)(C_s+2)}{2C_s(2C_s+1)}
\frac{2C_s+C_f}{(C_s+C_f+1)(D-1)}
+O(D^{-2}),
\end{align*}
and otherwise,
\begin{align*}
\CovWR{\sigma,\pi}{\theta_s}{\theta_{s+\Delta}}
&=-V_{\mathrm{ind}}
\frac{(C_s+2)(C_s+C_f)}
{(2C_s+1)(C_s+C_f+1)(D-1)}
+O(D^{-2}).
\end{align*}
If both $C_s/D$ and $C_f/D$ are bounded away from zero, then
\begin{align*}
\CovWR{\sigma,\pi}{\theta_s}{\theta_{s+\Delta}}
=
\left\{
\begin{array}{r l}
-V_{\mathrm{ind}}\left(
\frac{C_s(2D-C_s)+C_f(D-C_s)}
{(C_s+C_f)(2D-C_s)^2}\right)
+
O\left(\frac{1}{D^2}\right),
& D\text{ even and }\Delta = D/2, \\
-V_{\mathrm{ind}}
\hspace*{1.15cm}\left(\frac{1}{2D-C_s}\right)
\hspace*{1.0cm}+
O\left(\frac{1}{D^2}\right),
& \text{otherwise.}
\end{array}
\right.
\end{align*}
\end{lemma}

\begin{theorem}
\label{countThm}
Suppose we pre-process the location vector with a random permutation $\sigma$. Let the random variable $\theta_t \equiv $ number of $\locvecx$ seen before the first $\locvecO$ under the random permutation $\pi_{\rightarrow t}$. Let $K>1$, and let $\lambda_{D,K}$ be defined as in Definition~\ref{lambda_defn}. If $C_s$ and $C_f$ are fixed as $D\rightarrow\infty$, then
\begin{align*}
\VarWR{\sigma,\pi}{\sum_{t=1}^K \frac{\theta_t}{K}}
&=V_K\left[1-\frac{1}{D-1}\left\{
(K-1-\lambda_{D,K})
\frac{(C_s+2)(C_s+C_f)}{(2C_s+1)(C_s+C_f+1)}\right.\right.\\
&\hspace*{2.2cm}\left.\left.
+\lambda_{D,K}
\frac{(C_s+1)(C_s+2)}{2C_s(2C_s+1)}
\frac{2C_s+C_f}{C_s+C_f+1}
\right\}\right] +O(D^{-2}).\\
\end{align*}
If both $C_s/D$ and $C_f/D$ are bounded away from zero, then
\begin{align*}
\VarWR{\sigma,\pi}{\sum_{t=1}^K \frac{\theta_t}{K}}
&=V_K\left[1-\frac{K-1}{2D-C_s}\left\{
1-\frac{\lambda_{D,K}C_fD}
{(K-1)(C_s+C_f)(2D-C_s)}\right\}\right]+O(D^{-2}).
\end{align*}
\end{theorem}

\begin{proof}
The proof is similar to Theorem~\ref{ourownMAINtheorem}, where we add up the covariance terms.
\end{proof}

Each term within the square brackets in Theorem~\ref{countThm} is $< 1$. Hence the circulant permutation estimate gives a lower variance asymptotically. Substituting $C_s = f_2$ and $C_f = f_1-a$ gives the variance approximation of $m_t$.

To verify the approximations, we run experiments on synthetic data at $D = 128$ with varying $f_1,f_2,a$ values over 10000 iterations, where we similarly overlay the approximate theoretical variance when $\frac{C_s +C_f}{D} = \frac{f_1+f_2-a}{D}$ do not go to zero with the empirical MSEs in Figure~\ref{fig3CMH}. Unlike Figure~\ref{fig1CMH}, where the approximate theoretical variance overlaps with the empirical MSE, we see that at $f_1 = 18, f_2 = 32$, we see the theoretical variance curves do not overlap the empirical MSEs. However, how good is our variance approximation for $\frac{C_s+C_f}{D} \rightarrow 0$ if $D$ is large, and $f_1, f_2$ are small? 

\begin{figure*}[h]
\includegraphics[width=1.0\linewidth]{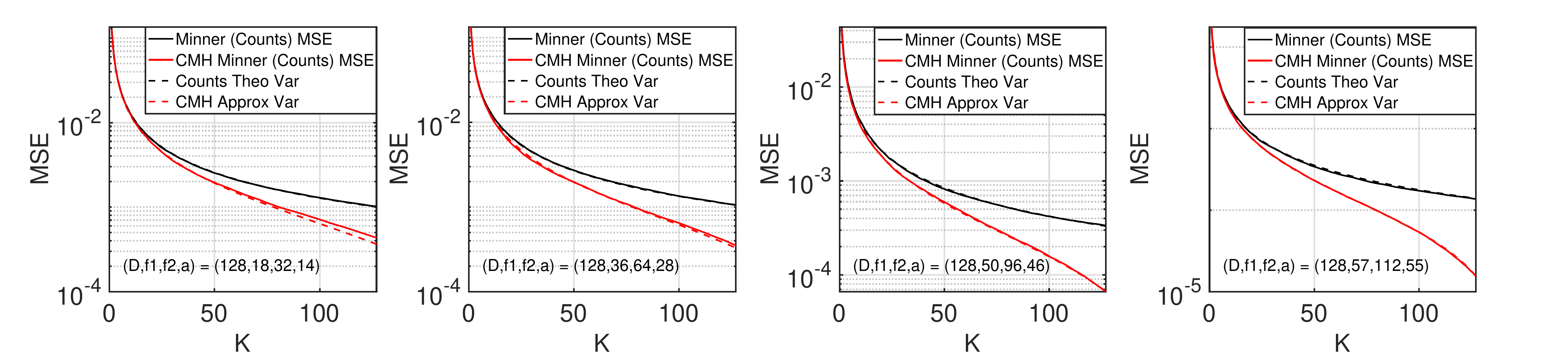}
\caption{Empirical MSE vs theoretical approximate variance in Theorem~\ref{countThm} for count data. The dotted lines overlap almost perfectly  for the last three plots. \label{fig3CMH}}
\end{figure*}

In Appendix~\ref{appendixBExpts}, we run another experiment averaged over 10000 iterations where we fix $(f_1,f_2,a) = (6,24,6)$, and increase this tuple up to $(f_1,f_2,a) = (96,128,132)$, with $D \in\{ 512, 1024, 2048, 4096\}$. In general, as $D$ increases, both our approximation variance ratios match the empirical MSEs. Finally, we  do experiments with the Words dataset as in Section~\ref{sec2} (Figure~\ref{words_data_M_CMH} in Appendix B). For the word pairs we choose, both theoretical approximate variance curves overlap the empirical MSE curves perfectly, even for the {\tt HONG-KONG} word pair, where $\frac{f_1-a}{D} = \frac{41}{65536}$ is sparse.

\section{Discussion and Conclusion} \label{sec5}

Each observation $\theta_s$ recorded by the circulant permutation estimator is correlated with another observation $\theta_{t}$, due to the circulant shift. Thus in order to theoretically compute (or compare) the variance of such an estimate, the corresponding covariance terms $\CovWR{\pi}{\theta_s}{\theta_t} = \ExptWR{\pi}{\theta_s\theta_t} - \ExptWR{\pi}{\theta_s}\ExptWR{\pi}{\theta_t}$ need to be computed. Computing $\ExptWR{\pi}{\theta_s\theta_t}$ is the first hurdle in circulant permutation algorithms, yet is needed to compute the variances for such estimators.

But by viewing MinHash under the paradigm of the draw order under sampling without replacement, we simplify the proof to compute $\ExptWR{\pi}{\theta_s\theta_t}$ in \cite{li2021c} via our approach in Lemma~\ref{lemma24}. Moreover, by using indicator functions and considering exchangeability, we can further compute $\ExptWR{\pi}{\theta_s\theta_t}$ for counts as well (our Lemma~\ref{lastlemA}), showing our approach applies beyond the original CMinHash setting.

The second hurdle with circulant permutations is finding the variance of the corresponding estimate. While Theorem 3.1 in \cite{li2021c} states a combinatorial expression for the variance, one equality in its proof does not hold in general (see Appendix~\ref{appendixA2SPAM}). Even apart from this point, the expression is not easily tractable, nor easy to analyze or compare against other MinHash algorithms, e.g. \cite{li2021c} had Theorem 3.4, Proposition 3.6, and many, many experiments in their paper just to compare CMinHash with MinHash (without even comparing against other MinHash algorithms). Researchers building on the circulant permutation framework would face similar technical difficulties when comparing variances.

Our approach suggests that exact accuracy is not needed, particularly because MinHash algorithms are used when $D$ is large in the first place. Lemma~\ref{taylorLemma} shows that by letting $D$ be large, we can do a straightforward computation to compute approximations of the variance. We demonstrate how to compute an approximation when set sizes are sparse relative to $D$, and when they are non sparse. Both approximations are accurate for practical purposes, allowing theoretical comparison of variances with other MinHash algorithms in the large $D$ scenario. 

Furthermore, our approach is not limited to CMinHash and applies to other MinHash estimators based on circulant permutation. The location vector setup of $\locvecO, \locvecx, \locvecm$ can be characterized as: ``{\bf stop} at the $s^{\text{th}}$ $\locvecO$, {\it do something} at a $\locvecx$, and continue at a $\locvecm$". Our proof approach in Lemma~\ref{lemma24} can be modified to a ``{\bf stop} at the $s^{\text{th}}$ $\locvecO$" scenario, and our approach with indicator random variables in Lemma~\ref{lastlemA} can account for a wide range of {\it do somethings}, by {\it conditioning on the element $\locvecx$ itself}, providing a direct way to compute expectations and variance for such estimators.

Thus, we believe our work makes circulant permutation estimators easier to analyze and extend. While state of the art algorithms are important, it is equally important to be able to understand and extend these algorithms. We provide a clear method to compute expectations and variance, and our proof techniques can be used directly to extend circulant permutations for other MinHash algorithms.

{\bf Limitations:} Generally, $D$ needs to be large for accurate approximations. Our experiments show that even $D = 65536$ (words data) gives good approximations. In practice, $D$ is much larger (e.g. $2^{32}$).

\begin{ack}
We would like to thank Bucknell's Presidential Fellow program, which paid a stipend to Avery Hood working on this paper. 
\end{ack}

\newpage

\bibliographystyle{plainnat}
\bibliography{reference_file.bib}

\newpage
\appendix

\section{Appendix: CMinHash and Alternative Proofs}
\label{appendixA}

\subsection{The CMinHash$(\sigma,\pi)$ algorithm}
\label{appendixA1ALGO}

CMinHash($\sigma,\pi$) is simple to implement, as it just requires a small tweak from Broder's original MinHash. We place Algorithm~\ref{broderAlgApp} and Algorithm~\ref{CMHAlgApp} side by side for comparison. The pre-processing stage for CMinHash (and by extension, algorithms using circulant permutation) consists of lines 1-4 in Algorithm~\ref{CMHAlgApp}, which replaces the independent permutation within the loop (e.g. Line 6 in Algorithm~\ref{broderAlgApp}).

\DontPrintSemicolon 
\SetKwInput{KwIn}{Input}
\SetKwInput{KwOut}{Output}

\begin{minipage}[t]{0.48\textwidth}
\begin{algorithm}[H]
\caption{Broder's MinHash \label{broderAlgApp}}

\KwIn{$\aSet{S}_1, \aSet{S}_2 \subseteq \{1,2,\hdots,D\}$\\
\phantom{{\bf Input:}} $\MHEst = 0$}
\KwOut{$\MHEst$}

\phantom{a}

\phantom{b}

\phantom{c}

\phantom{d}

\For{$t=1$ \KwTo $K$}{
  \text{Generate permutation } $\pi_t: [D]\rightarrow[D]$ \\
  $\theta_t \leftarrow \indic{\min\{\pi_t(\aSet{S}_1)\} = \min\{\pi_t(\aSet{S}_2)\}}$ \\
  $\MHEst \mathrel{+\!=} \theta_t$\;
} 
$\MHEst \leftarrow \MHEst/K$
\end{algorithm}
\end{minipage}
\hfill
\begin{minipage}[t]{0.48\textwidth}
\begin{algorithm}[H]
\caption{CMinHash($\sigma,\pi$) \label{CMHAlgApp}}

\KwIn{$\aSet{S}_1, \aSet{S}_2 \subseteq \{1,2,\hdots,D\}$\\
\phantom{{\bf Input:}} $\CMHEst = 0$}
\KwOut{$\CMHEst$}

\text{Generate permutation } $\sigma: [D]\rightarrow[D]$

$\aSet{S}_1 \leftarrow \sigma(\aSet{S}_1)$

$\aSet{S}_2 \leftarrow \sigma(\aSet{S}_2)$

\text{Generate permutation } $\pi: [D]\rightarrow[D]$

\For{$t=1$ \KwTo $K$}{
\phantom{d}

$\theta_t \leftarrow \indic{\min\{\pi_{\rightarrow t}(\aSet{S}_1)\} = \min\{\pi_{\rightarrow t}(\aSet{S}_2)\}}$ \\
  $\CMHEst \mathrel{+\!=} \theta_t$\;
} 
$\CMHEst \leftarrow \CMHEst/K$
\end{algorithm}
\end{minipage}

As both $\Expt{\MHEst}$ and $\Expt{\CMHEst}$ are unbiased estimators of $J$ due to linearity of expectation, the main result of \cite{li2021c} is to show that $\Var{\CMHEst} < \Var{\MHEst}$ for $K < D$ (always satisfied, since if $K>D$, computing $J$ directly is more efficient), provided $0 < J < 1$.

This is a surprising and impactful result, since $K$ in Broder's MinHash is the number of independent permutations, while $K$ in CMinHash$(\sigma,\pi)$ is the number of circular shifts of {\bf one} permutation. This means that the two independent permutations $\sigma,\pi$ can lead to a lower variance compared to $K$ independent permutations, and storing these two independent permutations in memory is cheaper than storing $K$ independent permutations. Conventional thought assumes that because the circulant shift is used for one permutation $\pi$, the variance necessarily needs to be higher.

However, we have not seen many papers building on the work of  \cite{li2021c} (e.g. using circulant permutations for {\it other} MinHash type algorithms), and we hypothesize this is due to the technical difficulty of the approaches in the paper, e.g. deriving a variance theorem equivalent to Theorem 3.1 in \cite{li2021c}, and then showing that the theoretical variance is indeed lower than other benchmark estimates. 

\subsection{Lemmas, Theorems and Propositions in the CMinHash$(\sigma,\pi)$ paper}
\label{appendixA2SPAM}

We give the original lemmas, theorems and propositions used in the seminal CMinHash paper, since we refer to them in the main text, and only modify them slightly to fit our notation.

We start with the following definition.

\begin{definition} Definition 2.2 in \cite{li2021c} \label{lidef222}\\
For $A,B \in \{\locvecO, \locvecx, \locvecm\}$, let $\{(A,B)~|~\Delta\}$ be the set $\{(i,j): (x_i,x_j) = (A,B), j-i = \Delta\}$. For $1 \leq \Delta \leq K-1$, define
\begin{align*}
\begin{array}{c c c}
\mathcal{L}_0(\Delta) = \{(\locvecO,\locvecO)|\Delta\},~~ & \mathcal{L}_1(\Delta) = \{(\locvecO,\locvecx)|\Delta\}, ~~& \mathcal{L}_2(\Delta) = \{(\locvecO,\locvecm)|\Delta\}, \\
\mathcal{G}_0(\Delta) = \{(\locvecm,\locvecO)|\Delta\},~~ & \mathcal{G}_1(\Delta) = \{(\locvecm,\locvecx)|\Delta\}, ~~& \mathcal{G}_2(\Delta) = \{(\locvecm,\locvecm)|\Delta\}, \\
\mathcal{H}_0(\Delta) = \{(\locvecx,\locvecO)|\Delta\},~~ & \mathcal{H}_1(\Delta) = \{(\locvecx,\locvecx)|\Delta\}, ~~& \mathcal{H}_2(\Delta) = \{(\locvecx,\locvecm)|\Delta\},
\end{array}
\end{align*}
with the following constraints for all $1 \leq \Delta \leq K-1$,
\begin{align*}
|\mathcal{L}_0| + |\mathcal{L}_1| + |\mathcal{L}_2| = |\mathcal{L}_0| + |\mathcal{G}_0| + |\mathcal{H}_0| &= a,  \\
|\mathcal{G}_0| + |\mathcal{G}_1| + |\mathcal{G}_2| = |\mathcal{L}_2| + |\mathcal{G}_2| + |\mathcal{H}_2| &= D - f, \\
|\mathcal{H}_0| + |\mathcal{H}_1| + |\mathcal{H}_2| = |\mathcal{L}_1| + |\mathcal{G}_1| + |\mathcal{H}_1| &= f-a. 
\end{align*}
\end{definition}

With this definition, \cite{li2021c} gave an expression for $\Var{\CMHEst}$.

\begin{theorem}
\label{lithm31long}
Theorem 3.1 in \cite{li2021c} (as stated) \\
Suppose $0 < a < f \leq D$ and $J \not\in\{0,1\}$. Then
\begin{align*}
\Var{\CMHEst} = \frac{J}{K} + \frac{(K-1)\widetilde{\mathcal{E}}}{K}-J^2,
\end{align*}
where $l = \max(0,D-2f+a)$ and
\begin{align*}
\widetilde{\mathcal{E}} & = \sum_{\Xi} \left( \frac{l_0}{f+g_0+g_1} + \frac{a(g_0+l_2)}{(f+g_0+g_1)f}  \right)\left( \sum_{s=l}^{D-f-1}\frac{\nck{D-f}{s}\nck{f-a-1}{D-f-s-1}}{\nck{D-a-1}{D-f-1}} \right. \\
& \left.  \times \frac{\nck{s}{n_1}\nck{D-f-s}{n_2}\nck{D-f-s}{n_3}\nck{f-a-(D-f-s)}{n_4}\nck{a-1}{a-l_1-l_2}}{\nck{D-1}{a}}    \right),
\end{align*}

with the feasible set $\Xi = \{l_0,l_2,g_0,g_1\}$ satisfies the constraints in Definition~\ref{lidef222} and 
\begin{align*}
n_1 & = g_0 - (D-f-s-g_1), \\
n_2 & = D-f-s-g_1, \\
n_3 & = l_2 - g_0 + (D-f-s-g_1), \\
n_4 & = l_1 - (D-f-s-g_1).
\end{align*}
\end{theorem}

The proof of this theorem can be found in \cite{li2021c}.

\begin{tcolorbox}
{\bf Remark:} During the rebuttal period of NeurIPS 2026, we wanted to further convince reviewers that our approach for the Minner estimator (Section~\ref{sec4}) was still simpler to understand, compared to the combinatorial expression we would get, following Li's approach. We never were able to replicate Li's approach for the Minner estimator, and we thought this was because we did not understand enough combinatorics. \\

Since ChatGPT 5.6 SOL was released, we asked it to apply Li's approach to the Minner estimator, in order to craft a better rebuttal for reviewers. Unfortunately, SOL found an error in Theorem 3.1, invalidating their approach. \\ 

The following italics is  output (slightly modified) from SOL, and the authors take full responsibility for it.\\

{\it In the proof of Theorem 3.1, it is stated that $\mathbb{E}_{\sigma,\pi}\left({\indicS{s}\indicS{s+\Delta}}\right)$ is the same for all $1 \leq \Delta \leq K -1$ because $\sigma$ is random. This does not hold in general. For example, take $D = 4, K =3, \mathcal{S}_1 = \{1,2\}, \mathcal{S}_2= \{1\}$. Then $a= 1, f=2$ and $J =1/2$. After fixing the position of $\locvecO$, the three possible location vectors are $(\locvecO,\locvecx, \locvecm,\locvecm)$, $(\locvecO, \locvecm, \locvecx, \locvecm)$ and $(\locvecO,\locvecm,\locvecm,\locvecx)$. We compute the expectation and variance by first principles. By Lemma 1,
\begin{align*}
\mathbb{E}_{\sigma,\pi}\left({\indicS{1}\indicS{2}}\right) & = \frac{1}{3}\left( \frac{1}{6} + \frac{1}{4} + \frac{1}{6} \right) = \frac{7}{36}, \\
\mathbb{E}_{\sigma,\pi}\left({\indicS{1}\indicS{3}}\right) & = \frac{1}{3}\left( \frac{1}{4} + 0 + \frac{1}{4} \right) = \frac{1}{6},
\end{align*}
The variance is hence
\begin{align*}
\mathrm{Var}_{\sigma,\pi}(\CMHEst) & = \frac{J}{3} + \frac{2}{9}\left( 2 \mathbb{E}_{\sigma,\pi}\left({\indicS{1}\indicS{2}}\right) + \mathbb{E}_{\sigma,\pi}\left({\indicS{1}\indicS{3}}\right)  \right) - J^2 \\
 & = \frac{13}{324}.
\end{align*}
On the other hand, applying Theorem 3.1 gives $\mathrm{Var}_{\sigma,\pi}(\CMHEst) = 5/108$. 
}
\end{tcolorbox}

In contrast, the approach in our paper considers every value of $\Delta$. To be precise, the derivation in the proof of Theorem 3.1 uses $\Delta = 1$, but the variance also involves the corresponding expectations for $2 \leq \Delta \leq K-1$. Even apart from this point, the variance is expressed as a combinatorial sum under several constraints, which makes it difficult to compare theoretically against the variance given by other MinHash algorithms. Moreover, the current approach taken to derive the result in Theorem 3.1 would need to be replicated for other MinHash type algorithms using circulant permutation, which is not straightforward.

In fact, in order for \cite{li2021c} to show that $\Var{\CMHEst} < \Var{\MHEst}$, they state and prove the following lemma and theorem.

\begin{lemma}
Lemma 3.3 (Strict Increment) in \cite{li2021c} \\
Let $f > a > 0$ and $K$ be arbitrary and fixed. Denote $\widetilde{\mathcal{E}}_D$ as in Theorem 3.1, with $D$ as a parameter. Then $\widetilde{\mathcal{E}}_{D+1} > \widetilde{\mathcal{E}}_D$ for all $D \geq f$.
\end{lemma}

\begin{theorem}
Theorem 3.4 (Uniform Superiority) in \cite{li2021c} (rephrased slightly)\\
For any two sets $\aSet{S}_1, \aSet{S}_2$, with $J \neq 0$ or $J \neq 1$, and $K > 1$, it holds that $\Var{\CMHEst} < \Var{\MHEst}$.
\end{theorem}

Finally, the CMinHash$(\sigma,\pi)$ estimator has some nice properties, which \cite{li2021c} gives in the following two propositions.

\begin{proposition}
Proposition 3.2 (Symmetry) in \cite{li2021c} \\
$\Var{\CMHEst}$ is the same for the $(D,f,a)$ data pair and the $(D,f,f-a)$ data pair, for $0 \leq a \leq f \leq D$.
\end{proposition}

\begin{proposition}
Proposition 3.6 (Consistent Improvement) in \cite{li2021c} \\
Suppose $f$ is fixed. In terms of $a$, the variance ratio $\frac{\Var{\MHEst}}{\Var{\CMHEst}}$ is a constant for any $0 < a < f$.
\end{proposition}

{\bf Contribution Two} of our paper presents an alternative approach to Theorem 3.1, giving a very accurate approximation to $\Var{\CMHEst}$, which can be replicated across different MinHash type algorithms which we discuss in Section~\ref{sec4} and Section~\ref{sec5} (our {\bf Contribution 3}).

Moreover, the exact covariance identity in \eqref{exactCMHcovariance2} makes Theorem 3.4, Proposition 3.2, and Proposition 3.6 immediate, while the  approximations avoid the constrained combinatorial sum in Theorem 3.1. We discuss this in Section~\ref{sec2}.

\subsection{A Simpler Approach to Lemma 2.4} \label{appendixa3_SIMPLE24}

We give a simpler approach for the following lemma.

\setcounter{lemma}{0}
\begin{lemma}
\lilemma
\end{lemma}

Before we embark on a proof, we give intuition on how we prove this lemma, which helps if we want to generalize this approach beyond MinHash.

Following the sampling without replacement paradigm where a random permutation can be thought of as the ``draw order", then for a fixed $\Delta$, we can imagine drawing the indices $(x_s, x_{s+\Delta})$ in some random order. The first entry in the tuple $(A,B) = (x_s, x_{s+\Delta})$ then corresponds to what is seen sequentially when $\pi$ is shifted $s$ times to the right, and the second entry corresponds to what is seen when $\pi$ is shifted $s + \Delta$ times to the right. 

For example, take $D=K=5$, $s=1$ and $\Delta=4$. After applying $\sigma$, suppose that the location vector is $\vec{x}=(\locvecm,\locvecm,\locvecO,\locvecx,\locvecO)$. Suppose one random permutation $\pi$ draws the indices in the order $1,2,\hdots,5$ under $\pi_{\rightarrow s}$. Shifting this same permutation rightwards by another $4$ units gives $\pi_{\rightarrow s+\Delta}$, under which the indices are drawn in the order $5,1,2,3,4$. We take the first entry at the index drawn under $\pi_{\rightarrow s}$ and the second entry at the index $4$ units to its right, returning to index $1$ after index $5$. Thus Draw 1 gives $(x_1,x_5)$, Draw 2 gives $(x_2,x_1)$, and the remaining draws follow the same order. Table~\ref{toy_example_draws} gives all five tuples. We bold the first $\locvecO$ in each entry.

\begin{table}[h]
\begin{center}
\begin{tabular}{|c|c| c| c|} \hline
{\bf Draw} & {\bf First entry} & {\bf Second entry} & {\bf Tuple} \\
 & under $\pi_{\rightarrow s}$ & under $\pi_{\rightarrow s+\Delta}$ & \\ \hline
{\bf Draw 1} & $\locvecm$ & ${\bf \locvecO}$ & $(\locvecm,{\bf \locvecO})$ \\
& & & \\
{\bf Draw 2} & $\locvecm$ & $\locvecm$ & $(\locvecm,\locvecm)$ \\
& & & \\
{\bf Draw 3} & ${\bf \locvecO}$ & $\locvecm$ & $({\bf \locvecO},\locvecm)$ \\
& & & \\
{\bf Draw 4} & $\locvecx$ & $\locvecO$ & $(\locvecx,\locvecO)$ \\
& & & \\
{\bf Draw 5} & $\locvecO$ & $\locvecx$ & $(\locvecO,\locvecx)$ \\ \hline
\end{tabular}
\end{center}
\caption{Five successive tuples from one random permutation, for $D=5$, $s=1$ and $\Delta=4$. The first entry corresponds to $\pi_{\rightarrow s}$ and the second entry to $\pi_{\rightarrow s+\Delta}$.\label{toy_example_draws}}
\end{table}

The first $\locvecO$ in the first entry appears at Draw 3, before the first $\locvecx$ at Draw 4. The first $\locvecO$ in the second entry appears at Draw 1, before the first $\locvecx$ at Draw 5. Hence $\indicS{s}=\indicS{s+\Delta}=1$, so $\indicS{s}\indicS{s+\Delta}=1$ for this $\pi$. The first $\locvecO$ occurs at different draws for the two entries.

Each draw gives the next tuple from this same random permutation. The subscripts $s$ and $s+\Delta$ in the two indicators denote the two circulant shifts. Each Example column in Tables~\ref{examplethingya} and~\ref{examplethingyb} shows a sequence of tuple draws from one random permutation.

Since a $\locvecO$ in the location vector means that the hashed values collide, then $\ExptWR{\pi}{\indicS{s}\indicS{s+\Delta}} > 0$ means that both $\min\{\pi_{\rightarrow s}(\aSet{S}_1)\} = \min\{\pi_{\rightarrow s}(\aSet{S}_2)\}$ and $\min\{\pi_{\rightarrow s+\Delta}(\aSet{S}_1)\} = \min\{\pi_{\rightarrow s +\Delta}(\aSet{S}_2)\}$ must hold.

Hence, $\ExptWR{\pi}{\indicS{s}\indicS{s+\Delta}} > 0$ can be thought of as the {\bf probability of a sequence of draws} of the nine different types of $D$ tuples where the first $\locvecO$ appears before the first $\locvecx$ in both entries of the tuple, as the random permutation $\pi$ is a sequence of draws without replacement. We call such a sequence of draws {\bf a success}, otherwise a failure. Five examples of successes and failures are in Table~\ref{examplethingya}. We bold the first $\locvecO$ which occurs.

\begin{table}[h]
\begin{center}
\begin{tabular}{|c|c| c| c| c| c| c|} \hline
& {\bf Example 1} & {\bf Example 2} & {\bf Example 3} & {\bf Example 4} & {\bf Example 5}  \\ \hline
{\bf Draw 1} & $(\locvecm,\locvecm)$ & $(\locvecm,\locvecm)$ & $({\bf \locvecO},\locvecm)$ & $({\bf \locvecO},\locvecm)$ & $({\bf \locvecO},\locvecm)$\\
& & & & & \\
{\bf Draw 2} & $(\locvecm,{\bf \locvecO})$ & $(\locvecm,{\bf \locvecO})$ & $(\locvecx, \locvecx)$ & $({\locvecO},\locvecm)$ & $(\locvecm,\locvecm)$\\
& & & & &\\
{\bf Draw 3} & $({\bf \locvecO}, \locvecm)$ & $(\locvecm, \locvecx)$ & & $({\locvecO},\locvecm)$ & $(\locvecm,\locvecm)$  \\
& & & & & \\
{\bf Draw 4} & & $({\bf \locvecO}, \locvecm)$  & & $(\locvecx,{\bf \locvecO})$ & $(\locvecm,{\bf \locvecO})$ \\ \hline
\end{tabular}
\end{center}
\caption{Five examples of sequences of draws. Each example is a vertical column. \label{examplethingya}}
\end{table}

{\bf Example 1} and {\bf Example 2} are functionally the same (and are both successes). While {\bf Example 2} has a $\locvecx$ in the second entry at Draw 3, it does not matter, because the first $\locvecO$ in the second entry was already drawn in Draw 2.

{\bf Example 3} is an example of a failure, as a $\locvecx$ appears before an $\locvecO$ in the second entry at Draw 2. 

{\bf Example 4} and {\bf Example 5} are also functionally the same. We also do not worry if an $\locvecx$ appears next to an $\locvecO$ (as in Draw 4) for {\bf Example 4}.

In fact, since it does not matter if a $\locvecx$ appears after an $\locvecO$, we do not need to worry about elements appearing after the first $\locvecO$, and hence can view our examples as in Table~\ref{examplethingyb}.

By treating the tuples $(A,B)$ as sampling without replacement, and computing the appropriate probability of success, we can now prove Lemma~\ref{lemma24} easily.

\begin{table}[h]
\begin{center}
\begin{tabular}{|c|c| c| c| c| c| c|} \hline
& {\bf Example 1} & {\bf Example 2} & {\bf Example 3} & {\bf Example 4} & {\bf Example 5}  \\ \hline
{\bf Draw 1} & $(\locvecm,\locvecm)$ & $(\locvecm,\locvecm)$ & $({\bf \locvecO},\locvecm)$ & $({\bf \locvecO},\locvecm)$ & $({\bf \locvecO},\locvecm)$\\
& & & & & \\
{\bf Draw 2} & $(\locvecm,{\bf \locvecO})$ & $(\locvecm,{\bf \locvecO})$ & $(\phantom{\locvecO}, \locvecx)$ & $(\phantom{{\bf \locvecO}},\locvecm)$ & $(\phantom{\locvecm},\locvecm)$\\
& & & & &\\
{\bf Draw 3} & $({\bf \locvecO}, \phantom{\locvecm})$ & $(\locvecm, \phantom{\locvecx})$ & & $(\phantom{\bf \locvecO},\locvecm)$ & $(\phantom{\locvecm},\locvecm)$  \\
& & & & & \\
{\bf Draw 4} & & $({\bf \locvecO}, \phantom{\locvecm})$  & & $(\phantom{\locvecO},{\bf \locvecO})$ & $(\phantom{\locvecm},{\bf \locvecO})$ \\ \hline
\end{tabular}
\end{center}
\caption{Five examples of sequences of draws, where we do not record the type of entry after the first $\locvecO$ is drawn. Each example is a vertical column. \label{examplethingyb}}
\end{table}

\begin{proof}

Similar to Li's approach, we consider three cases, and condition on $\indicS{s}$. More specifically, we distinguish cases where the success for the {\bf first entry of the tuple} happens when we draw  $\{(\locvecO,\locvecx), (\locvecO,\locvecO),(\locvecO,\locvecm)\}$ (before any $\locvecx$ for the {\bf first entry}), and consider the probability for a success where where the first $\locvecO$ is drawn before a $\locvecx$ for the {\bf second entry of the tuple}. Table~\ref{proof_by_picture} gives a pictorial intuition of the five cases.

\begin{table}[h]
\begin{center}
\begin{tabular}{|c|c| c| c| c| c| c|} \hline
& {\bf Case 1} & {\bf Case 2a} & {\bf Case 2b} & {\bf Case 3a} & {\bf Case 3b}  \\ \hline
 & $\vdots$ & $\vdots$ & $\vdots$ & $\vdots$ & $\vdots$\\
& & & & & \\
 & $(\locvecm,{\bf \locvecO})$ & $(\locvecm,\locvecm)$ & $(\locvecm,{\bf \locvecO})$ & $(\locvecm,{\bf \locvecO})$ & $\vdots$\\
& & & & &\\
 & $\vdots$ & $\vdots$ & $\vdots$ & $\vdots$ & $\vdots$\\
& & & & &\\
 & $(\locvecm,\phantom{\locvecO})$ & $\vdots$ & $(\locvecm,\phantom{\locvecO})$ & $\vdots$ & $\vdots$\\
& & & & &\\
 & $\vdots$ & $\vdots$ & $\vdots$ & $\vdots$ & $\vdots$\\
& & & & &\\
{\bf Draw $t$} & $({\bf \locvecO}, \locvecx)$ & $({\bf \locvecO}, {\bf \locvecO})$ & $({\bf \locvecO}, \locvecO)$ & $({\bf \locvecO},\locvecm)$ & $({\bf \locvecO},\locvecm)$  \\
& & & & & \\
 &  & &  &  & $\vdots$\\
& & & & & \\
 & &   & &  & $(\phantom{\locvecO},{\bf \locvecm})$ \\ \hline
\end{tabular}
\end{center}
\caption{Pictorial intuition of the five (sub)cases of {\bf successes} we consider, as we condition on the first $\locvecO$ seen in the first entry of the tuple. A bold ${\bf \locvecO}$ denotes the first $\locvecO$ seen for that entry. A blank entry in a tuple means it does not matter what that entry is since an $\locvecO$ was drawn before.\label{proof_by_picture}}
\end{table}

Case One and Case Two follow along similar lines in their proof, but Case Three is vastly simplified. For each case, we show the product of probabilities ``sequentially" in our reasoning.

We remind the reader that tuples which do not affect the hash colliding are not involved in the below probability calculations. As an analogy, consider a bag with one blue marble, two red marbles, and $N$ green marbles. The probability of drawing a blue marble before a red marble is always $\frac{1}{3}$, regardless of $N$. 

{\bf Case One:} The first $\locvecO$ for the {\bf first entry of the tuple} is seen at $(\locvecO,\locvecx)$.  A success for the {\bf second entry of the tuple} can only happen before $(\locvecO, \locvecx)$, hence we need the probability of seeing the first $\locvecO$ appearing before any $\locvecx$ for the {\bf second entry of the tuple}. This is the product of the probabilities of i) drawing a $(\locvecm,\locvecO)$ before any invalid tuples in $\{(\locvecm,\locvecx), (\locvecx, \locvecO), (\locvecx, \locvecm), (\locvecx, \locvecx),(\locvecO, \locvecO), (\locvecO, \locvecm), (\locvecO,\locvecx)\}$, and ii) drawing a $(\locvecO,\locvecx)$ before any invalid tuples in $\{(\locvecx, \locvecO), (\locvecx, \locvecm), (\locvecx, \locvecx),(\locvecO, \locvecO), (\locvecO, \locvecm)\}$, which is given by
\begin{align}
p_1 & = \frac{|\mathcal{G}_0|}{f + |\mathcal{G}_0| + |\mathcal{G}_1|}  \frac{|\mathcal{L}_1|}{f} = \frac{ |\mathcal{G}_0||\mathcal{L}_1|}{(f + |\mathcal{G}_0| + |\mathcal{G}_1|)f}.  \label{p3_prob}
\end{align}

{\bf Case Two:} The first $\locvecO$ for the {\bf first entry of the tuple} is seen at $(\locvecO,\locvecO)$.  A success is always guaranteed once $(\locvecO,\locvecO)$ is drawn,
provided there are no invalid $\locvecx$ before an $\locvecO$ for the {\bf second entry of the tuple}. Hence there are two possible disjoint sequences ending with $(\locvecO,\locvecO)$ that are successes. 
\begin{itemize}
\item [a)] A sequence where no $(\locvecm,\locvecx)$ is seen before $(\locvecO,\locvecO)$, where we draw $(\locvecO,\locvecO)$ before invalid tuples $\{(\locvecm,\locvecx)$, $(\locvecx, \locvecO), (\locvecx, \locvecm), (\locvecx, \locvecx), (\locvecO, \locvecm), (\locvecO, \locvecx)\}$. The probability is given by
\begin{align}
p_2 & = \frac{|\mathcal{L}_0|}{f + |\mathcal{G}_1|}. \label{p1_prob}
\end{align}
\item [b)] A sequence where at least one $(\locvecm,\locvecx)$ is seen before $(\locvecO,\locvecO)$. This means a $(\locvecm,\locvecO)$ must be seen before it. This is the product of the probabilities of i) first drawing $(\locvecm,\locvecO)$ before invalid tuples in $\{(\locvecm,\locvecx), (\locvecx,\locvecO), (\locvecx,\locvecm), (\locvecx,\locvecx), (\locvecO,\locvecO), (\locvecO,\locvecm), (\locvecO,\locvecm)  \}$, ii) next drawing $(\locvecm,\locvecx)$ before invalid tuples in $\{(\locvecx,\locvecO), (\locvecx,\locvecm), (\locvecx,\locvecx), (\locvecO,\locvecO), (\locvecO,\locvecm), (\locvecO,\locvecx)  \}$, and iii) finally drawing $(\locvecO,\locvecO)$ before invalid tuples in $\{(\locvecx,\locvecO), (\locvecx,\locvecm), (\locvecx,\locvecx), (\locvecO,\locvecm), (\locvecO,\locvecx)  \}$, which is given by
\begin{align}
p_3 & = \frac{|\mathcal{G}_0|}{f + |\mathcal{G}_0| + |\mathcal{G}_1|}  \frac{|\mathcal{G}_1|}{f + |\mathcal{G}_1|}  \frac{|\mathcal{L}_0|}{f} = \frac{|\mathcal{G}_0||\mathcal{G}_1||\mathcal{L}_0|}{(f + |\mathcal{G}_0| + |\mathcal{G}_1|)(f + |\mathcal{G}_1|)f}. \label{p2_prob}
\end{align}
\end{itemize}

{\bf Case Three:} The first $\locvecO$ for the {\bf first entry of the tuple} is seen at $(\locvecO,\locvecm)$. Then a success occurs when either the first $\locvecO$ for the {\bf second entry of the tuple} appears before the first $(\locvecO,\locvecm)$, or it appears after, together with no $\locvecx$ appearing before the first $\locvecO$ for the {\bf first entry of the tuple}. Hence there are two possible disjoint sequences which are successes. 
\begin{itemize}
\item [a)] A sequence which ends at $(\locvecO,\locvecm)$, which means an $(\locvecm,\locvecO)$ must be seen before it. This is the product of probabilities of i) first drawing $(\locvecm,\locvecO)$ before any invalid tuples $\{(\locvecm,\locvecx), (\locvecx, \locvecO), (\locvecx, \locvecm), (\locvecx, \locvecx),(\locvecO, \locvecO),(\locvecO,\locvecm), (\locvecO,\locvecx)\}$, and ii) drawing $(\locvecO,\locvecm)$ before invalid tuples in $\{(\locvecx, \locvecm), (\locvecx, \locvecO), (\locvecx, \locvecx), (\locvecO,\locvecO),(\locvecO, \locvecx)\}$, given  by
\begin{align}
p_4 & = \frac{|\mathcal{G}_0|}{f + |\mathcal{G}_0| + |\mathcal{G}_1|}  \frac{|\mathcal{L}_2|}{f} = \frac{|\mathcal{G}_0||\mathcal{L}_2|}{(f + |\mathcal{G}_0| + |\mathcal{G}_1|)f}  \label{p4_prob}
\end{align}
\item [b)] A sequence which ends after $(\locvecO,\locvecm)$, which means the first $\locvecO$ for the {\bf second entry of the tuple} appears after it. This is the product of probabilities of i) first drawing $(\locvecO,\locvecm)$ before invalid tuples $\{(\locvecm,\locvecO),(\locvecm,\locvecx), (\locvecx, \locvecO), (\locvecx, \locvecm), (\locvecx, \locvecx),(\locvecO, \locvecO), (\locvecO,\locvecx)\}$, and ii) drawing an element in $\{(\locvecO,\locvecO), (\locvecm,\locvecO), (\locvecx,\locvecO)\}$ before invalid tuples in $\{(\locvecO,\locvecx), (\locvecm,\locvecx), (\locvecx,\locvecx)\}$, given by
\begin{align}
p_5 & = \frac{|\mathcal{L}_2|}{f + |\mathcal{G}_0| + |\mathcal{G}_1|} \frac{a}{f} = \frac{|\mathcal{L}_2|a}{(f + |\mathcal{G}_0|+|\mathcal{G}_1|)f}.  \label{p5_prob}
\end{align}
\end{itemize}
Adding \eqref{p3_prob}, \eqref{p1_prob}, \eqref{p2_prob}, \eqref{p4_prob}, \eqref{p5_prob} gives $\frac{|\mathcal{L}_0| + (|\mathcal{G}_0| + |\mathcal{L}_2|)J}{f + |\mathcal{G}_0| + |\mathcal{G}_1|}$.

\end{proof}
\subsection{An Alternative Approach to Theorem 3.1} \label{appendix_a4_LI}

\setcounter{lemma}{1}

\begin{lemma}
Suppose we pre-process the location vector with the random permutation $\sigma$. Let $B_D=2D-f-1$, and suppose $0<J<1$. For all $1\leq\Delta\leq K-1$, if $f/D\rightarrow0$, then
\begin{align*}
\ExptWR{\sigma,\pi}{\indicS{s}\indicS{s+\Delta}}
&=\left\{\begin{array}{r l}
J^2-\dfrac{J(1-J)f}{2(f-1)(D-1)}+O\left(\dfrac{f}{D^2}\right), & D\text{ even and }\Delta=D/2,\\
J^2-\dfrac{J(1-J)f}{(2f-1)(D-1)}+O\left(\dfrac{f}{D^2}\right), & \text{otherwise},
\end{array}\right.\\
\CovWR{\sigma,\pi}{\indicS{s}}{\indicS{s+\Delta}}
&=\left\{\begin{array}{r l}
-\dfrac{J(1-J)f}{2(f-1)(D-1)}+O\left(\dfrac{f}{D^2}\right), & D\text{ even and }\Delta=D/2,\\
-\dfrac{J(1-J)f}{(2f-1)(D-1)}+O\left(\dfrac{f}{D^2}\right), & \text{otherwise}.
\end{array}\right.
\end{align*}
If $f/D$ is bounded away from zero, then
\begin{align*}
\ExptWR{\sigma,\pi}{\indicS{s}\indicS{s+\Delta}}
&=J^2-\frac{J(1-J)}{B_D}+O(D^{-2}),\\
\CovWR{\sigma,\pi}{\indicS{s}}{\indicS{s+\Delta}}
&=-\frac{J(1-J)}{B_D}+O(D^{-2}).
\end{align*}
\end{lemma}

\label{friendly_notation_reminder}
\begin{proof}
Lemma~\ref{lemma24} gives $\ExptWR{\pi}{\indicS{s}\indicS{s+\Delta}}$. Define
\begin{align}
N_{\Delta} & = \Card{\mathcal{L}_0(\Delta)} +J \left( \Card{\mathcal{G}_0(\Delta)} + \Card{\mathcal{L}_2(\Delta)}\right), \label{NMnumerator}\\
M_{\Delta} & = f + \Card{\mathcal{G}_0(\Delta)} + \Card{\mathcal{G}_1(\Delta)}. \label{NMdenominator}
\end{align}
Thus
\begin{align*}
\ExptWR{\pi}{\indicS{s}\indicS{s+\Delta}} = \frac{N_\Delta}{M_\Delta}. 
\end{align*}

We first condition on the $f$ positions of the location vector whose entries are $\locvecO$ or $\locvecx$. Let
\begin{align}
X_{\Delta} = \sum_{i=1}^D \indic{x_i \neq \locvecm,~x_{i+\Delta} \neq \locvecm}.
\label{Xdefinition}
\end{align}
The quantity $X_{\Delta}$ counts the tuples $(x_i,x_{i+\Delta})$ for which both entries are $\locvecO$ or $\locvecx$. There are $f-X_\Delta$ tuples whose first entry is $\locvecm$ and whose second entry is $\locvecO$ or $\locvecx$. To see this, each of the $f$ entries equal to $\locvecO$ or $\locvecx$ has one entry at distance $\Delta$ before it, and $X_{\Delta}$ of those preceding entries are also $\locvecO$ or $\locvecx$. Therefore
\begin{align}
\Card{\mathcal{G}_0(\Delta)} + \Card{\mathcal{G}_1(\Delta)} =f - X_{\Delta}. \label{CrossingCount}
\end{align}
Substitution of \eqref{CrossingCount} into \eqref{NMdenominator} gives
\begin{align}
M_\Delta = 2f - X_{\Delta}. \label{MfromX}
\end{align}

Conditional on the positions counted by \eqref{Xdefinition}, the $a$ entries equal to $\locvecO$ are placed among the $f$ positions equal to $\locvecO$ or $\locvecx$ by sampling without replacement. For a tuple counted by $X_{\Delta}$, the probability that both entries are $\locvecO$ is $a(a-1)/(f(f-1))$. Also, $f - X_{\Delta}$ tuples go from $\locvecm$ to either $\locvecO$ or $\locvecx$, and another $f-X_{\Delta}$ tuples go from either $\locvecO$ or $\locvecx$ to $\locvecm$. In each case, the entry which is not $\locvecm$ equals $\locvecO$ with probability $a/f$. Hence
\begin{align}
\ExptWR{\sigma}{\Card{\mathcal{ L}_0(\Delta)}\mid X_{\Delta}} & = X_{\Delta}\frac{a(a-1)}{f(f-1)}, \label{ConditionalLzero} \\
\ExptWR{\sigma}{\Card{\mathcal{G}_0(\Delta)} + \Card{\mathcal{L}_2(\Delta)}\mid X_{\Delta}} & = 2(f-X_{\Delta})\frac{a}{f}.
\label{ConditionalCrossingO}
\end{align}
Using $J=a/f$ in \eqref{NMnumerator},  \eqref{ConditionalLzero} and \eqref{ConditionalCrossingO} give
\begin{align}
\ExptWR{\sigma}{N_{\Delta} \mid X_{\Delta}} = J^2(2f - X_{\Delta}) - \frac{J(1-J)}{f-1}X_{\Delta}.
\label{ConditionalNfromX}
\end{align}
The denominator in \eqref{MfromX} is fixed once $X_\Delta$ is fixed. Dividing \eqref{ConditionalNfromX} by \eqref{MfromX} and then averaging over the positions gives
\begin{align}
\ExptWR{\sigma,\pi}{\indicS{s} \indicS{s+\Delta}} & = J^2 - \frac{J(1-J)}{f-1} \ExptWR{\sigma}{\frac{X_{\Delta}}{2f-X_{\Delta}}}, \label{exactCMHcovariance} \\
\CovWR{\sigma,\pi}{\indicS{s}}{\indicS{s+\Delta}} & = -\frac{J(1-J)}{f-1} \ExptWR{\sigma}{\frac{X_{\Delta}}{2f -X_{\Delta}}}.
\label{exactCMHcovariance2}
\end{align}

We next approximate the expectation in \eqref{exactCMHcovariance2}. Write $Y_i=\indic{x_i\neq \locvecm,\ x_{i+\Delta}\neq \locvecm}$, so that $X_{\Delta} = \sum_iY_i$ by \eqref{Xdefinition}. For each tuple $(x_i,x_{i+\Delta})$, the probability that both entries are $\locvecO$ or $\locvecx$ is $f(f-1)/(D(D-1))$. Summing this probability over the $D$ possible values of $i$ gives
\begin{align}
\ExptWR{\sigma}{X_{\Delta}} = D\frac{f(f-1)}{D(D-1)} = \frac{f(f-1)}{D-1}. \label{expectedAdjacent}
\end{align}

Suppose first that $D$ is odd or $\Delta\neq D/2$. In the sum $X_\Delta (X_\Delta-1) = \sum_{i\neq j} Y_iY_j$, there are $2D$ choices of $(i,j)$ for which the tuples $(x_i,x_{i+\Delta})$ and $(x_j,x_{j+\Delta})$ have one common position. The remaining $D(D-3)$ choices use four distinct positions. If $u^{\underline r} = u(u-1)\cdots(u-r+1)$, sampling without replacement therefore gives
\begin{align}
\ExptWR{\sigma}{X_{\Delta}(X_{\Delta}-1)} & = 2D\frac{f^{\underline 3}}{D^{\underline 3}} + D(D-3)\frac{f^{\underline 4}}{D^{\underline 4}} =\frac{f(f-1)^2(f-2)}{(D-1)(D-2)}. \label{expectedTwoAdjacent}
\end{align}
For every integer $0\leq x\leq f$,
\begin{align}
\frac{x}{2f-x} =\frac{x}{2f-1} + \frac{x(x-1)}{(2f-1)(2f-x)}. \label{GenericFractionIdentity}
\end{align}
Apply \eqref{GenericFractionIdentity} to $X_\Delta$, use \eqref{expectedAdjacent} for the first term, and use $2f-X_\Delta\geq f$ together with \eqref{expectedTwoAdjacent} for the second term. This gives
\begin{align}
\frac{1}{f-1}\ExptWR{\sigma}{\frac{X_{\Delta}}{2f-X_{\Delta}}} & = \frac{f}{(2f-1)(D-1)}+R_{\Delta}, \label{GenericSparseExpansion} \\
0 \leq R_{\Delta} & \leq \frac{(f-1)(f-2)}{(2f-1)(D-1)(D-2)} = O\left(\frac{f}{D^2}\right). \label{GenericSparseRemainder}
\end{align}

Now suppose that $D$ is even and $\Delta = D/2$. The tuple $(x_i,x_{i+D/2})$ is accompanied by the tuple $(x_{i+D/2}, x_i)$, so $X_{\Delta}$ is even. Let $Q_{\Delta} = X_{\Delta}/2$. Two different pairs of these tuples have four different positions. There are $(D/2)(D/2-1)$ choices of the first and second pair, and hence
\begin{align}
\ExptWR{\sigma}{X_{\Delta}(X_{\Delta}-2)} & = 4\ExptWR{\sigma}{Q_{\Delta}(Q_{\Delta}-1)} \notag \\
 & = 4\frac{D}{2}\left(\frac{D}{2}-1\right)
\frac{f^{\underline 4}}{D^{\underline 4}} = \frac{f(f-1)(f-2)(f-3)}{(D-1)(D-3)}. \label{expectedTwoDoverTwo}
\end{align}
For every even integer $0\leq x\leq f$,
\begin{align}
\frac{x}{2f-x} = \frac{x}{2(f-1)} + \frac{x(x-2)}{2(f-1)(2f-x)}. \label{DoverTwoFractionIdentity}
\end{align}
Equations \eqref{expectedAdjacent}, \eqref{expectedTwoDoverTwo}, and \eqref{DoverTwoFractionIdentity} give
\begin{align}
\frac{1}{f-1}\ExptWR{\sigma}{\frac{X_{\Delta}}{2f-X_{\Delta}}} & = \frac{f}{2(f-1)(D-1)} + R_{D/2}, \label{DoverTwoSparseExpansion} \\
0 \leq R_{D/2} & \leq \frac{(f-2)(f-3)}{2(f-1)(D-1)(D-3)} = O\left(\frac{f}{D^2}\right). \label{DoverTwoSparseRemainder}
\end{align}
Substituting \eqref{GenericSparseExpansion} and \eqref{DoverTwoSparseExpansion}, with their respective bounds \eqref{GenericSparseRemainder} and \eqref{DoverTwoSparseRemainder}, into \eqref{exactCMHcovariance} and \eqref{exactCMHcovariance2} proves the two expressions when $f/D \rightarrow 0$.
The two bounds hold for every $1\leq\Delta\leq K-1$ because the pair count only distinguishes the case in \eqref{expectedTwoAdjacent} from the case in \eqref{expectedTwoDoverTwo}.

Now suppose that $f/D$ is bounded away from zero. Define
\begin{align*}
h(x) = \frac{x}{2f-x}, ~~~~\mu_{\Delta} = \ExptWR{\sigma}{X_{\Delta}} =\frac{f(f-1)}{D-1}, 
\end{align*}
where the value of $\mu_\Delta$ is from \eqref{expectedAdjacent}. Direct algebra gives
\begin{align}
\frac{h(\mu_{\Delta})}{f-1} = \frac{1}{2D-f-1} = \frac{1}{B_D}. \label{densePlugIn}
\end{align}
The same pair count used in \eqref{expectedTwoAdjacent} and \eqref{expectedTwoDoverTwo} gives
\begin{align}
\VarWR{\sigma}{X_\Delta} = \left\{\begin{array}{r l}
\dfrac{2f(D-f)(f-1)(D-f-1)}{(D-1)^2(D-3)}, & D\text{ even and }\Delta=D/2, \\
\dfrac{f(D-f)(f-1)(D-f-1)}{(D-1)^2(D-2)}, & \text{otherwise}.
\end{array}\right.
\label{DenseXvariance}
\end{align}
For $0\leq x\leq f$,
\begin{align}
h''(x) = \frac{4f}{(2f-x)^3},~~~~ 0 < h''(x) \leq \frac{4}{f^2}.
\label{DenseSecondDerivative}
\end{align}
We Taylor expansion about $\mu_{\Delta}$, take expectations, and using the fact $\ExptWR{\sigma}{X_{\Delta} - \mu_{\Delta}}=0$ and \eqref{DenseSecondDerivative} gives
\begin{align}
0 \leq \ExptWR{\sigma}{h(X_{\Delta})} -h(\mu_{\Delta}) \leq \frac{2}{f^2} \VarWR{\sigma}{X_\Delta}. \label{DenseTaylorBound}
\end{align}
If $f/D$ is bounded below by a positive constant, the two expressions for the variance in \eqref{DenseXvariance} are $O(D)$. Dividing \eqref{DenseTaylorBound} by $f-1$ therefore gives
\begin{align}
0 \leq \frac{\ExptWR{\sigma}{h(X_{\Delta})} -h(\mu_{\Delta})}{f-1} \leq \frac{2\VarWR{\sigma}{X_{\Delta}}}{f^2(f-1)} = O(D^{-2}).
\label{DenseScaledTaylorBound}
\end{align}
Combining \eqref{densePlugIn} and \eqref{DenseScaledTaylorBound} gives
\begin{align}
\frac{1}{f-1}\ExptWR{\sigma}{\frac{X_{\Delta}}{2f - X_{\Delta}}} = \frac{1}{B_D}+O(D^{-2}). \label{DenseRatioExpansion}
\end{align}
Finally, substituting \eqref{DenseRatioExpansion} into \eqref{exactCMHcovariance} and \eqref{exactCMHcovariance2} gives the two expressions when $f/D$ is bounded away from zero.
\end{proof}

\setcounter{theorem}{1}

\begin{theorem}\label{ourownMAINtheorem}
Let $0 <a \leq f \leq D$. Suppose we preprocess the location vector with the random permutation $\sigma$, and generate an independent permutation $\pi$. Let $B_D=2D-f-1$. If $f/D\rightarrow0$, 
\begin{align*}
\VarWR{\sigma,\pi}{\CMHEst}
=\frac{J(1-J)}{K}\left[1-\frac{f}{D-1}\left\{\frac{K-1-\lambda_{D,K}}{2f-1}+\frac{\lambda_{D,K}}{2(f-1)}\right\}\right]
+O\left(\frac{f}{D^2}\right).
\end{align*}
If $f/D$ is bounded away from zero, then
\begin{align*}
\VarWR{\sigma,\pi}{\CMHEst}
=\frac{J(1-J)}{K}\left(1-\frac{K-1}{B_D}\right)+O(D^{-2}).
\end{align*}
\end{theorem}

\begin{proof}
This follows from expanding $\VarWR{\sigma,\pi}{\sum_{t=1}^K\indicS{t}/K} = \frac{1}{K^2} \sum_{s=1}^K\sum_{t=1}^K \CovWR{\sigma,\pi}{\indicS{s}}{\indicS{t}},$ using covariance terms in Lemma~\ref{taylorLemma}. The term $\lambda_{D,K}$ accounts for the proportion of covariance terms where $D$ is even and $\Delta=D/2$.
\end{proof}

\subsection{Additional Experiments to Verify Variance Approximation}\label{appendix_a5_EXPT}

While Figure~\ref{fig1CMH} shows that our variance approximation for $\frac{f}{D}\rightarrow 0$ does match the empirical MSE even at small $D = 128$, we also need to understand how it behaves at edge cases (i.e. when $f_1$, $f_2$ approaches zero). While this usually isn't the case in practical settings, we want to investigate how well we can rely on the given approximation.

Hence, we plot $(f,a)$ pairs of $(2,1), (4,2), (6,3), (8,4)$ at $D = 128$ and $D= 1024$ respectively. We see that at $f = 2, 4$, the approximate theoretical variance lines are slightly off, but still extremely close to the empirical MSE, but at $f = 6, 8$, both lines overlap each other almost perfectly. Hence, we believe our approximation of the theoretical variance is practical enough to do comparisons.

\begin{figure*}[h]
\includegraphics[width=1.0\linewidth]{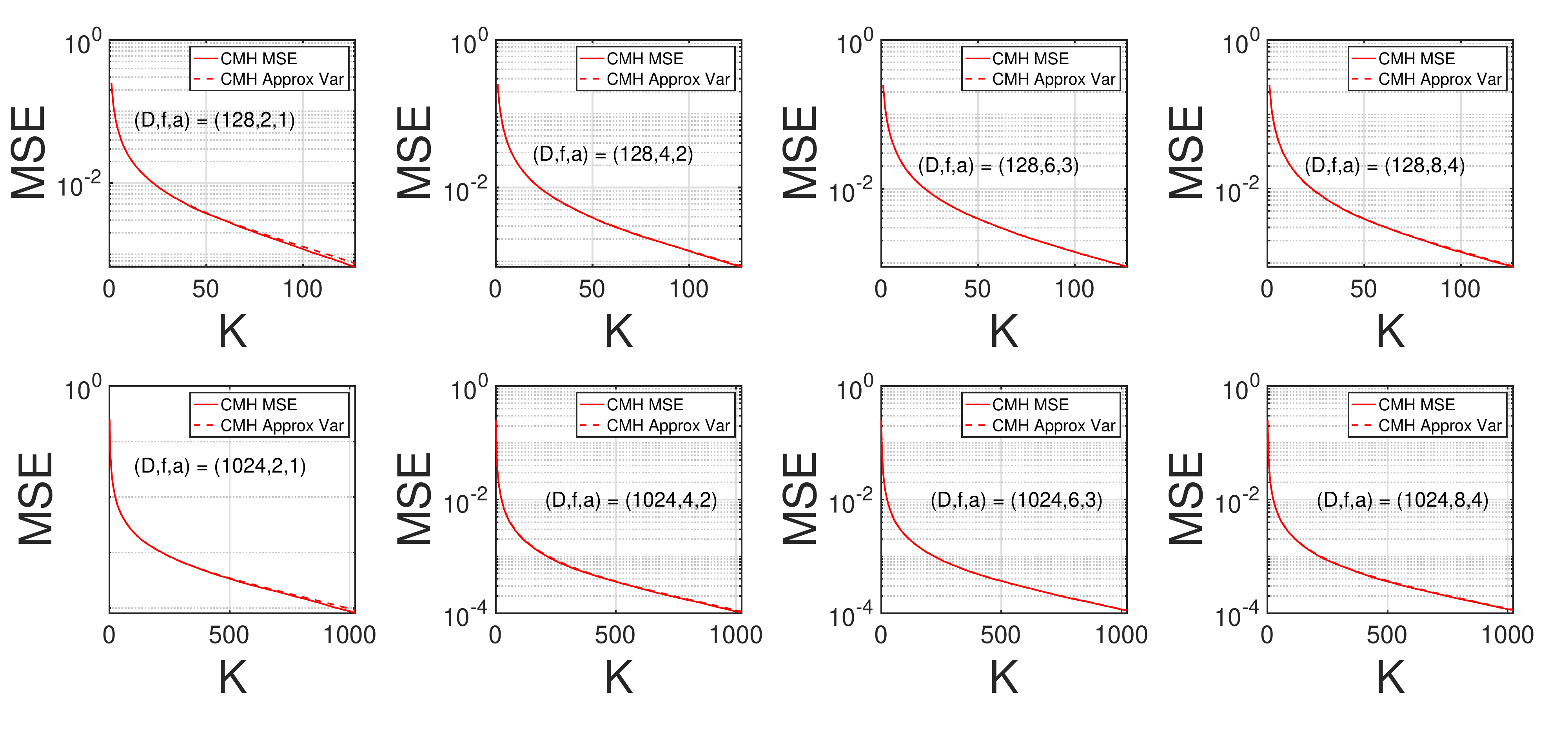}
\caption{Empirical MSE vs theoretical approximate variance in Theorem~\ref{ourownMAINtheorem} of MinHash and CMinHash $(\sigma,\pi)$, at edge cases when $f$ is small. At $f=2,4$, the dotted lines do not overlap perfectly, but they do at $f = 6, 8$. \label{fig2CMH}}
\end{figure*}

Finally, we plot both theoretical approximations for the variance in Thereom~\ref{ourownMAINtheorem} and the empirical MSE for select word pairs from the words dataset \citep{li2005using} in Figure~\ref{words_data_CMH}. Here, $D = 65536$, while the word pairs have $\frac{f_i}{D} \in [0.0143,0.3226]$, i.e. where $\frac{f}{D} \rightarrow 0$ and where $\frac{f}{D}$ does not tend to zero.

We average across 10,000 iterations as well. We can see that both theoretical approximations as well as the empirical MSE lines overlap perfectly.

\begin{figure*}[h]
\begin{center}
\includegraphics[width=0.8\linewidth]{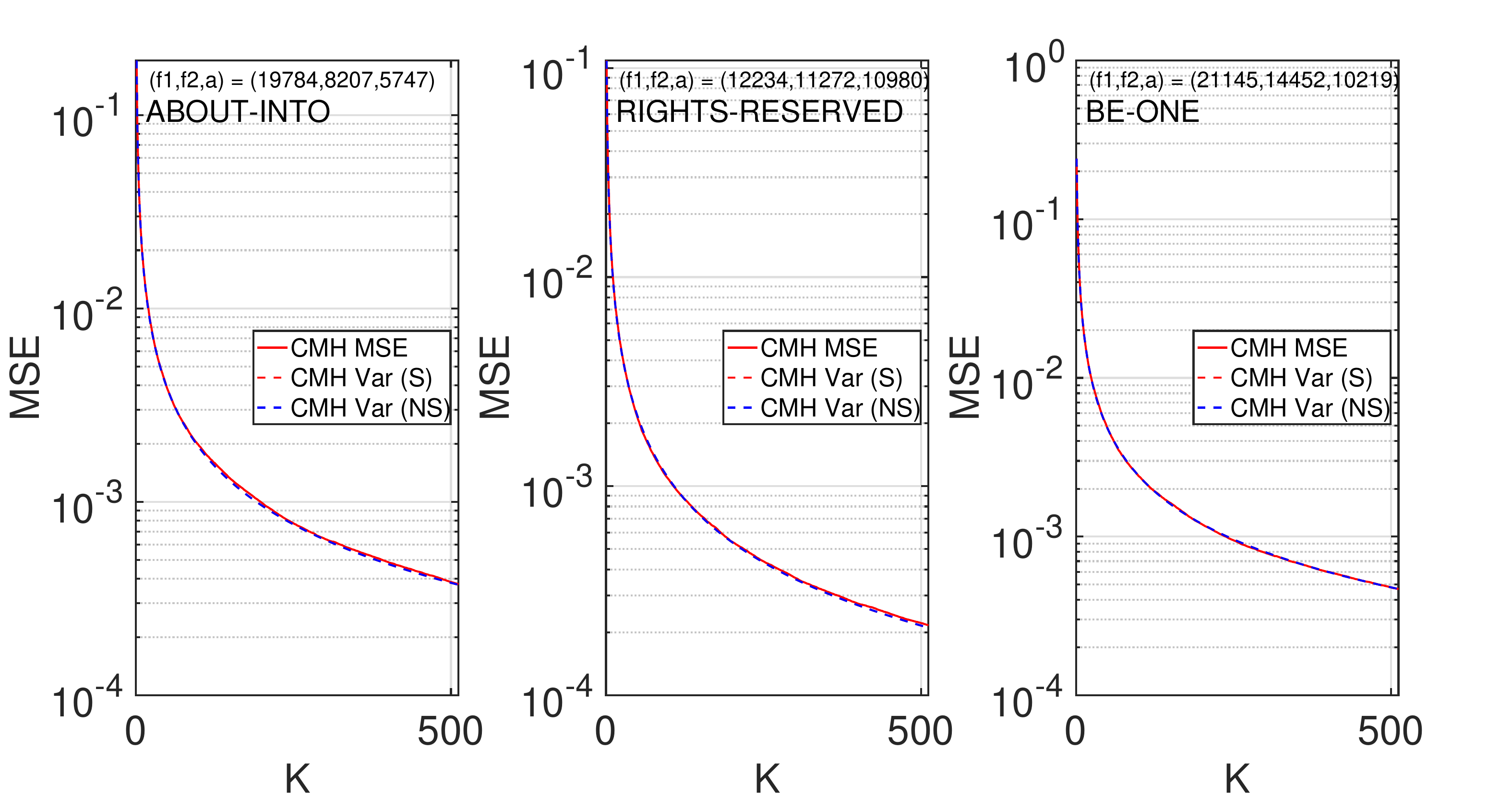}

\includegraphics[width=0.8\linewidth]{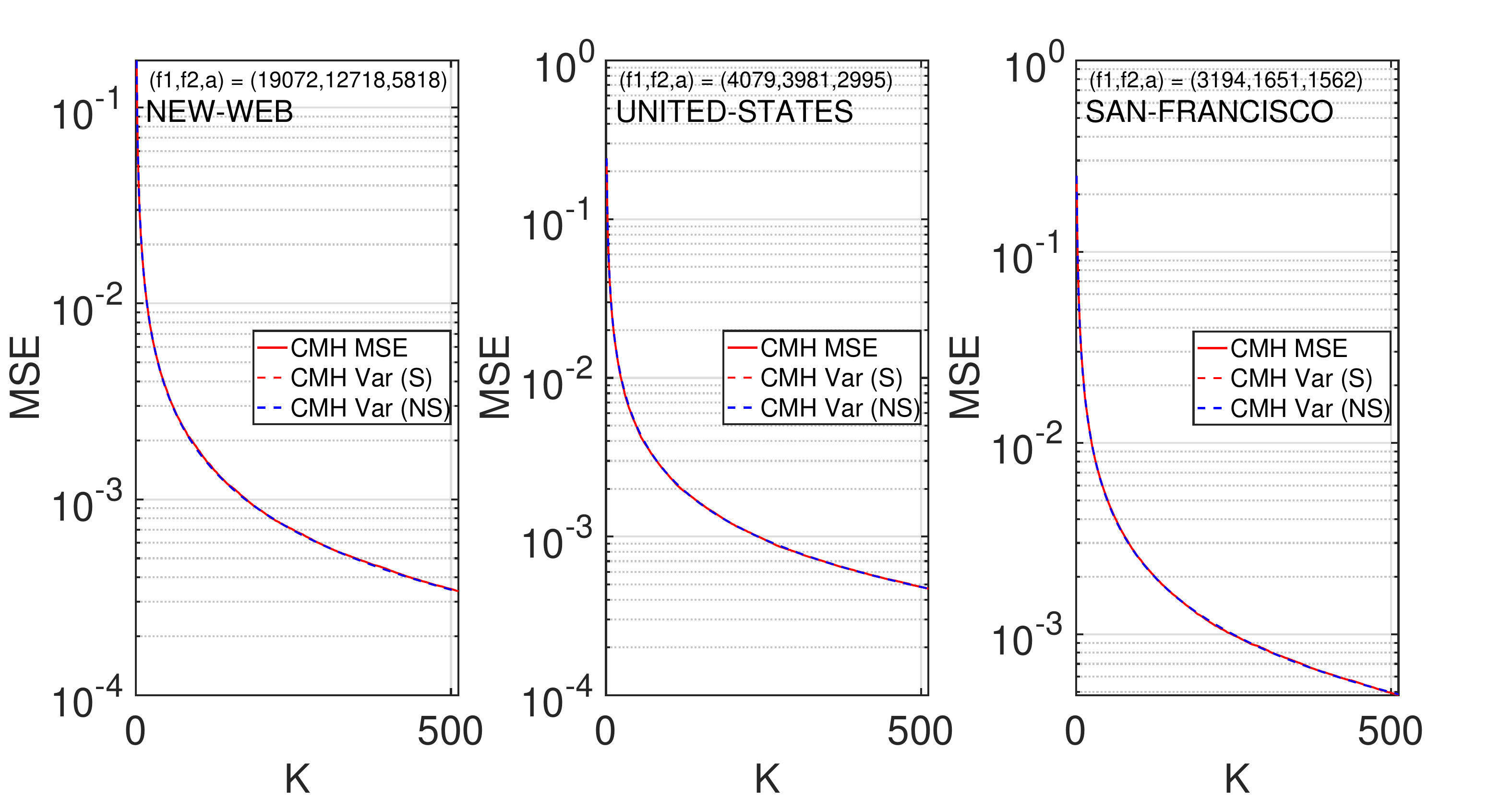}

\includegraphics[width=0.8\linewidth]{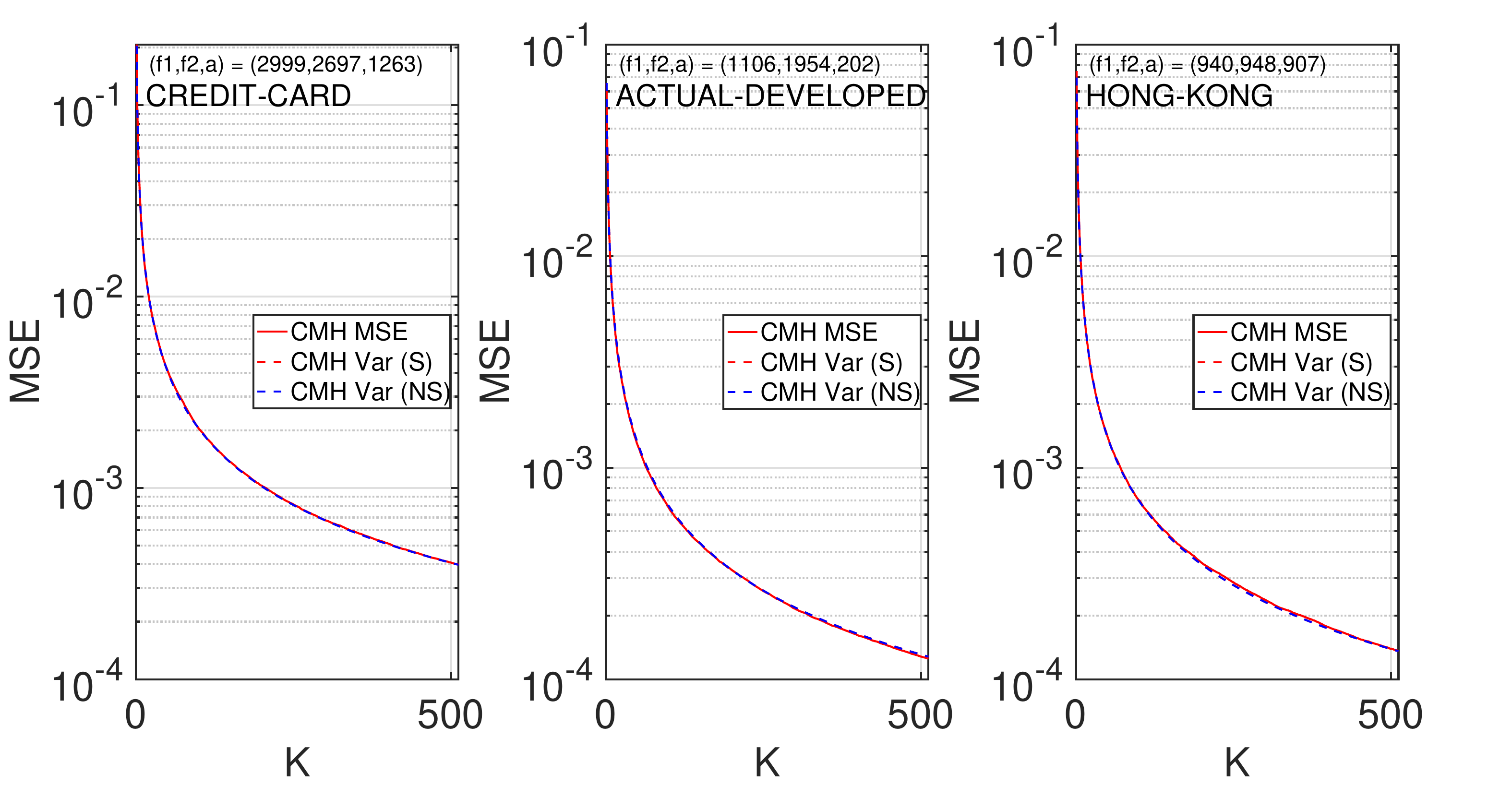}
\end{center}
\caption{Empirical MSE versus both theoretical approximate variances in Theorem~\ref{ourownMAINtheorem} of MinHash and CMinHash $(\sigma,\pi)$ for selected word pairs. CMH Var (S) is the the {\it sparse} case where $\frac{f_1}{D} \rightarrow 0$. CMH Var (NS) is the the {\it non-sparse} case where $\frac{f_1}{D}$ does not tend to zero.  Both theoretical approximations and empirical MSE lines overlap perfectly. \label{words_data_CMH}}
\end{figure*}

\clearpage

\section{Appendix: Circulant Permutation for The Minner Estimators}
\label{appendixB}
This appendix covers Section~\ref{sec4} in more detail.

\subsection{Understanding the Minner Estimators} \label{appendixB_understandingMinner}

To better understand the Minner estimators, we give the following lemma before explaining in detail.

\setcounter{lemma}{4}
\begin{lemma}\label{emptylemma}
Suppose we have a random permutation $\pi$ on $[D]$. Let $y = \indic{ \min(\pi(\aSet{S}_2)) \in\pi(\aSet{S}_1) }$, and $m$ to be the count of elements in $\pi(\aSet{S}_1)$ strictly lesser than $\min(\pi(\aSet{S}_2))$. Then $\Expt{y} = \frac{a}{f_2}$, $\Expt{m} = \frac{f_1-a}{f_2+1}$, $\Var{m} = \frac{(f_1-a)(f_1+f_2-a+1)f_2}{(f_2+1)^2(f_2+2)}$.
\end{lemma}
\setcounter{figure}{3}
\begin{figure}[h]
\begin{center}
\begin{align*}
    \begin{array}{c | c c c c c c c c c c c c c} 
    \text{\bf Draw Order} & 1 & 2 & 3 & 4 & 5 & 6 & 7 & 8  \\ \hline
      \text{Set 1} & \circ & \bullet & \circ & \bullet & \bullet & \bullet & \circ & \bullet  \\
      \text{Set 2} &  \circ & \circ & \circ & \circ & \circ & \circ & \circ & \boxed{\bullet} \\
    \end{array}
\end{align*}
\end{center}
\caption{We have $\min\{\pi_t(\aSet{S}_2)\} = 8$, with $y_t = 1$, which can be thought of as the presence (or absence) of a $\bullet$ at draw 8 for $\aSet{S}_1$, and $m_t = 4$, which can be thought of as counting how many $\bullet$ were seen in previous draws 1-7 for $\aSet{S}_1$.}
\end{figure}

We also replicate Figure~\ref{fig4eg} here, as viewing the random variables under the ``drawing objects" from the sampling without replacement paradigm makes the proof straightforward.
\setcounter{figure}{8}
\begin{proof}
Set $\aSet{F}_{+} = \left\{\bulletbullet\right\}$, and $\aSet{F}_{-} = \left\{\circbullet\right\}$, with $\Card{\aSet{F}_{+}} = a$ and $\Card{\aSet{F}_{-}} = f_2-a$. Then $y$ is simply the probability of seeing an element from $\aSet{F}_{+}$ drawn before an element in $\aSet{F}_{-}$, which is $\frac{a}{f_2-a+a} = \frac{a}{f_2}$.

Now, set $\aSet{F}_{+} = \left\{ \bulletbullet, \circbullet \right\}$, and $\aSet{F}_{-} = \left\{\bulletcirc \right\}$, with $\Card{\aSet{F}_{+}} = f_2$ and $\Card{\aSet{F}_{-}} = f_1-a$. Then $m$ is the count of the number of elements in $\aSet{F}_{-}$ seen before the first element in $\aSet{F}_{+}$.

Suppose we index each  element in $\aSet{F}_{-}$ by $i$, $1 \leq i \leq f_1-a$. We can write the random variable $m$ as
\begin{align*}
m = \sum_{i=1}^{f_1-a} \indic{\text{element}~i~\text{is drawn before any element in $\mathcal{F}_+$}}.
\end{align*}
There are $f_2$ elements in $\mathcal{F}_+$, and we require the single element $i$ in $\mathcal{F}_{-}$ to be drawn first, hence 
\begin{align*}
\Expt{\indic{\text{element}~i~\text{is drawn before any element in $\mathcal{F}_+$}}} = \frac{1}{\Card{\aSet{F}_{+}} + 1} = \frac{1}{f_2+1},
\end{align*}
by exchangeability. Now by linearity of expectations, $\Expt{m} = \sum_{i=1}^{f_1-a}\frac{1}{f_2+1} = \frac{f_1-a}{f_2+1}$.

To find $\Var{m}$, we need to find $\Expt{m^2}$. Certainly,
\begin{align*}
\Expt{m^2} & = \sum_{i=1}^{f_1-a} \Expt{\indic{\text{element}~i~\text{is drawn before any element in $\mathcal{F}_+$}}} \\
 & \phantom{=...} + \sum_{i \neq j}^{f_1-a} \Expt{\indic{\text{element}~i~\text{is drawn before any element in $\mathcal{F}_+$}}\indic{\text{element}~j~\text{is drawn before any element in $\mathcal{F}_+$}}}.
\end{align*}
Now $\Expt{\indic{\text{element}~i~\text{is drawn before any element in $\mathcal{F}_+$}}\indic{\text{element}~j~\text{is drawn before any element in $\mathcal{F}_+$}}} = \frac{2}{(f_2+1)(f_2+2)}$, as this can be thought of as out of the $f_2$ elements in $\mathcal{F}_{+}$, and the two elements $i,j$ the $i,j$ elements must be drawn first. Therefore,
\begin{align*}
\Expt{m^2} & = \frac{f_1-a}{f_2 + 1} + \frac{2(f_1-a)(f_1-a-1)}{(f_2+1)(f_2+2)},
\end{align*}
and
\begin{align*}
\Var{m} & = \Expt{m^2} - \Expt{m}^2 = \frac{(f_1-a)f_2(f_1+f_2-a+1)}{(f_2+1)^2(f_2+2)}.
\end{align*}
\end{proof}

We remark that both these estimators have different performance based on the sizes of the sets $\aSet{S}_1, \aSet{S}_2$. However, it is not in the scope of this paper to analyze which estimator is better given $a,f_1,f_2,D$, or propose better Minner estimators. The paper focuses on {\bf extending the circulant permutation framework when counts are involved}.

\subsection{Proof of Lemma 3} \label{appendixB_proofLemma4}

We recall the definition of the nine sets again, but this time using our notation $C_s, C_f$ and $C_n$ (success, failure, neutral) for generalizability.

For $A,B \in \{\locvecO, \locvecx, \locvecm\}$, let $\{(A,B)~|~\Delta\}$ be the set $\{(i,j): (x_i,x_j) = (A,B), j-i = \Delta\}$. For $1 \leq \Delta \leq K-1$, define
\begin{align*}
\begin{array}{c c c}
\mathcal{L}_0(\Delta) = \{(\locvecO,\locvecO)|\Delta\},~~ & \mathcal{L}_1(\Delta) = \{(\locvecO,\locvecx)|\Delta\}, ~~& \mathcal{L}_2(\Delta) = \{(\locvecO,\locvecm)|\Delta\}, \\
\mathcal{G}_0(\Delta) = \{(\locvecm,\locvecO)|\Delta\},~~ & \mathcal{G}_1(\Delta) = \{(\locvecm,\locvecx)|\Delta\}, ~~& \mathcal{G}_2(\Delta) = \{(\locvecm,\locvecm)|\Delta\}, \\
\mathcal{H}_0(\Delta) = \{(\locvecx,\locvecO)|\Delta\},~~ & \mathcal{H}_1(\Delta) = \{(\locvecx,\locvecx)|\Delta\}, ~~& \mathcal{H}_2(\Delta) = \{(\locvecx,\locvecm)|\Delta\},
\end{array}
\end{align*}
with the following constraints for all $1 \leq \Delta \leq K-1$
\begin{align*}
|\mathcal{L}_0| + |\mathcal{L}_1| + |\mathcal{L}_2| = |\mathcal{L}_0| + |\mathcal{G}_0| + |\mathcal{H}_0| &= C_s,  \\
|\mathcal{G}_0| + |\mathcal{G}_1| + |\mathcal{G}_2| = |\mathcal{L}_2| + |\mathcal{G}_2| + |\mathcal{H}_2| &= C_n,  \\
|\mathcal{H}_0| + |\mathcal{H}_1| + |\mathcal{H}_2| = |\mathcal{L}_1| + |\mathcal{G}_1| + |\mathcal{H}_1| &= C_f.
\end{align*}

\setcounter{lemma}{2}
\begin{lemma}
Suppose we pre-process the location vector with a random permutation $\sigma$. Let the random variable $\theta_s \equiv $ number of $\locvecx$ seen before the first $\locvecO$ under the random permutation $\pi_{\rightarrow s}$. Let $B_{\Delta} = 2C_s - \Card{\mathcal{L}_0} + 1$. Then for any $1 \leq s < t \leq K$, with $t - s = \Delta$,
\begin{align*}
\ExptWR{\pi}{\theta_s \theta_{s+\Delta}}
&=
\frac{\Card{\mathcal{H}_1}}{B_{\Delta}}
+
\frac{2C_f^2-C_f(\Card{\mathcal{L}_1}+\Card{\mathcal{H}_0})-2\Card{\mathcal{H}_1}}
{(C_s+1)(B_{\Delta}+1)} +
\frac{C_f(\Card{\mathcal{L}_1}+\Card{\mathcal{H}_0})-2\Card{\mathcal{L}_1}\Card{\mathcal{H}_0}}
{(C_s+1)B_{\Delta}(B_{\Delta}+1)}.
\end{align*}
If $C_s$ and $C_f$ are fixed as $D\rightarrow\infty$, then, when $D$ is even and $\Delta=D/2$,
\begin{align*}
\CovWR{\sigma,\pi}{\theta_s}{\theta_{s+\Delta}}
&=-V_{\mathrm{ind}}
\frac{(C_s+1)(C_s+2)}{2C_s(2C_s+1)}
\frac{2C_s+C_f}{(C_s+C_f+1)(D-1)}
+O(D^{-2}),
\end{align*}
and otherwise,
\begin{align*}
\CovWR{\sigma,\pi}{\theta_s}{\theta_{s+\Delta}}
&=-V_{\mathrm{ind}}
\frac{(C_s+2)(C_s+C_f)}
{(2C_s+1)(C_s+C_f+1)(D-1)}
+O(D^{-2}).
\end{align*}
If both $C_s/D$ and $C_f/D$ are bounded away from zero, then
\begin{align*}
\CovWR{\sigma,\pi}{\theta_s}{\theta_{s+\Delta}}
=
\left\{
\begin{array}{r l}
-V_{\mathrm{ind}}\left(
\frac{C_s(2D-C_s)+C_f(D-C_s)}
{(C_s+C_f)(2D-C_s)^2}\right)
+
O\left(\frac{1}{D^2}\right),
& D\text{ even and }\Delta = D/2, \\
-V_{\mathrm{ind}}
\hspace*{1.15cm}\left(\frac{1}{2D-C_s}\right)
\hspace*{1.0cm}+
O\left(\frac{1}{D^2}\right),
& \text{otherwise.}
\end{array}
\right.
\end{align*}
\end{lemma}

\begin{proof}
Our proof will follow a similar approach to the proof of Lemma~\ref{lemma24} with respect to sequence of draws.

For ease of notation, we define the sets
\begin{align*}
\aSet{O}_1 & = \mathcal{L}_0 \cup \mathcal{L}_1 \cup \mathcal{L}_2 ~\text{(tuples where first entry is $\locvecO$),} \\
\aSet{O}_2 & = \mathcal{L}_0 \cup \mathcal{G}_0 \cup \mathcal{H}_0 ~\text{(tuples where second entry is $\locvecO$),} \\
\aSet{X}_1 & = \mathcal{H}_0 \cup \mathcal{H}_1 \cup \mathcal{H}_2 ~\text{(tuples where first entry is $\locvecx$),} \\
\aSet{X}_2 & = \mathcal{L}_1 \cup \mathcal{G}_1 \cup \mathcal{H}_1 ~\text{(tuples where second entry is $\locvecx$),} \\
\end{align*}
and define the variable
\begin{align*}
U=\Card{\aSet{O}_1\cup\aSet{O}_2}=2C_s-\Card{\aSet{L}_0}.
\end{align*}

For an element $i$ in $\aSet{X}_1$, we write the random variable $I_i^{(1)} \equiv \indic{i~\text{appears before elements in}~\aSet{O}_1}$. Similarly, for an element $j$ in $\aSet{X}_2$, we write $I_j^{(2)} \equiv \indic{j~\text{appears before elements in}~ \aSet{O}_2}$.
The two random variables can therefore be written as
\begin{align}
\theta_s = \sum_{i\in\aSet{X}_1}I_i^{(1)}, ~~~~ \theta_{s+\Delta} = \sum_{j\in\aSet{X}_2}I_j^{(2)}.
\label{CountThetaSums}
\end{align}
Multiplying the two sums in \eqref{CountThetaSums} gives
\begin{align}
\ExptWR{\pi}{\theta_s \theta_{s+\Delta}} & = \sum_{i\in\aSet{X}_1\cap\aSet{X}_2} \ExptWR{\pi}{I_i^{(1)} I_i^{(2)}} + \sum_{\substack{i\in\aSet{X}_1,\ j \in \aSet{X}_2\\i\neq j}} \ExptWR{\pi}{I_i^{(1)}I_j^{(2)}}. \label{countDecomp}
\end{align}

For the first sum in \eqref{countDecomp}, the equality $i = j$ requires the tuple $(\locvecx,\locvecx)$, so there are $\Card{\mathcal H_1}$ choices. For any one of them, $I_i^{(1)}I_i^{(2)} = 1$ when $i$ is the first tuple among $\aSet{O}_1 \cup \aSet{O}_2\cup\{i\}$. Exchangeability of these $U+1$ tuples gives probability $1/(U+1)$. Hence
\begin{align}
\sum_{i\in\aSet{X}_1\cap\aSet{X}_2} \ExptWR{\pi}{I_i^{(1)} I_i^{(2)}} = \frac{\Card{\mathcal H_1}}{U+1}. \label{lem3_cont1}
\end{align}

We now evaluate the second sum in \eqref{countDecomp}. As in the proof of Lemma~\ref{lemma24}, we consider four cases depending on whether $i$ is in $\aSet{O}_2$ and whether $j$ is in $\aSet{O}_1$.

{\bf Case One: Suppose $i\notin \aSet{O}_2$ and $j\notin \aSet{O}_1$.} Let $E$ be the earliest tuple seen out of $\aSet{O}_1\cup\aSet{O}_2\cup\{i\}\cup\{j\}$. Since both $i$ and $j$ are outside $\aSet{O}_1\cup\aSet{O}_2$, we have
\begin{align*}
\Prob{E=i}=\Prob{E=j}=\frac{1}{U+2}.
\end{align*}
A valid sequence occurs in two disjoint ways. First, $i$ is drawn before any tuple in $\aSet{O}_1\cup\aSet{O}_2\cup\{j\}$, and then $j$ is drawn before any tuple in $\aSet{O}_2$. Second, $j$ is drawn before any tuple in $\aSet{O}_1\cup\aSet{O}_2\cup\{i\}$, and then $i$ is drawn before any tuple in $\aSet{O}_1$. Thus the probability for each ordered pair in this case is
\begin{align*}
\frac{1}{U+2}\frac{1}{C_s+1}
+
\frac{1}{U+2}\frac{1}{C_s+1}
=
\frac{2}{(C_s+1)(U+2)}.
\end{align*}
There are
\begin{align*}
(\Card{\mathcal{H}_1}+\Card{\mathcal{H}_2})(\Card{\mathcal{G}_1}+\Card{\mathcal{H}_1})-\Card{\mathcal{H}_1}
=
(C_f-\Card{\mathcal{H}_0})(C_f-\Card{\mathcal{L}_1})-\Card{\mathcal{H}_1}
\end{align*}
ordered pairs with $i\notin \aSet{O}_2$, $j\notin \aSet{O}_1$, and $i\neq j$. Hence the contribution from Case One is
\begin{align}
\frac{2\left((C_f-\Card{\mathcal{H}_0})(C_f-\Card{\mathcal{L}_1})-\Card{\mathcal{H}_1}\right)}
{(C_s+1)(U+2)}. \label{lem3_cont2}
\end{align}

{\bf Case Two: Suppose $i\in\aSet{O}_2$ and $j\notin\aSet{O}_1$.} Since $i\in\aSet{X}_1\cap\aSet{O}_2$, this means $i\in\mathcal{H}_0$. A valid sequence must first draw $j$ before any tuple in $\aSet{O}_2$, and then draw $i$ before any tuple in $\aSet{O}_1$. Since only $j$ is added to $\aSet{O}_1\cup\aSet{O}_2$, the probability for each ordered pair in this case is
\begin{align*}
\frac{1}{U+1}\frac{1}{C_s+1}
=
\frac{1}{(C_s+1)(U+1)}.
\end{align*}
There are $\Card{\mathcal{H}_0}(C_f-\Card{\mathcal{L}_1})$ such ordered pairs. Hence the contribution from Case Two is
\begin{align}
\frac{\Card{\mathcal{H}_0}(C_f-\Card{\mathcal{L}_1})}{(C_s+1)(U+1)}. \label{lem3_cont3}
\end{align}

{\bf Case Three:} Suppose $i\notin\aSet{O}_2$ and $j\in\aSet{O}_1$. By the same reasoning as Case Two, $j\in\mathcal{L}_1$, and the contribution is
\begin{align}
\frac{(C_f-\Card{\mathcal{H}_0})\Card{\mathcal{L}_1}}{(C_s+1)(U+1)}. \label{lem3_cont4}
\end{align}

{\bf Case Four: Suppose $i\in\aSet{O}_2$ and $j\in\aSet{O}_1$.} This case contributes zero. Since $i\in\aSet{O}_2$, the event $\indic{j~\text{before}~\aSet{O}_2}=1$ requires $j$ to appear before $i$. However, since $j\in\aSet{O}_1$, the event $\indic{i~\text{before}~\aSet{O}_1}=1$ requires $i$ to appear before $j$.

Adding \eqref{lem3_cont1}, \eqref{lem3_cont2}, \eqref{lem3_cont3}, and \eqref{lem3_cont4}, and substituting $U=2C_s-\Card{\mathcal{L}_0}$, gives
\begin{align}
\ExptWR{\pi}{\theta_s \theta_{s+\Delta}}
&=
\frac{\Card{\mathcal{H}_1}}{2C_s-\Card{\mathcal{L}_0}+1}
+
\frac{2C_f^2-C_f(\Card{\mathcal{L}_1}+\Card{\mathcal{H}_0})-2\Card{\mathcal{H}_1}}
{(C_s+1)(2C_s-\Card{\mathcal{L}_0}+2)} \notag \\
&\phantom{=}
+
\frac{C_f(\Card{\mathcal{L}_1}+\Card{\mathcal{H}_0})-2\Card{\mathcal{L}_1}\Card{\mathcal{H}_0}}
{(C_s+1)(2C_s-\Card{\mathcal{L}_0}+1)(2C_s-\Card{\mathcal{L}_0}+2)}.
\label{epicount}
\end{align}

This proves the first expression in Lemma~\ref{lastlemA}.

Now suppose we pre-process the location vector with the random permutation $\sigma$. Then we want to find the expectation of \eqref{epicount} with respect to $\sigma$. For ease of notation, let
\begin{align}
Z = (Z_0,Z_1,Z_2,Z_3) = \left(\Card{\mathcal{L}_0},\Card{\mathcal{L}_1}, \Card{\mathcal {H}_0}, \Card{\mathcal{H}_1}\right),
\label{CountZdefinition}
\end{align}
and let $F(Z)$ denote the right hand side of \eqref{epicount}. When $Z=(0,0,0,0)$,
\begin{align}
F(0,0,0,0) = \frac{2C_f^2}{(C_s+1)(2C_s+2)} = \left(\frac{C_f}{C_s+1}\right)^2. \label{CountBaseline}
\end{align}

First suppose that $C_s$ and $C_f$ are fixed while $D \rightarrow \infty$, and that $D$ is odd or $\Delta\neq D/2$. Let
\begin{align*}
R_\Delta = Z_0 + Z_1 + Z_2 + Z_3.
\end{align*}
Thus $R_\Delta$ is the number of tuples $(x_i,x_{i+\Delta})$ whose two entries are $\locvecO$ or $\locvecx$. When $R_\Delta = 1$, the four possible values of $Z$ and their changes from \eqref{CountBaseline} are
\begin{align}
\begin{array}{c|c|c}
\text{pair} & Z & F(Z) -F(0,0,0,0) \\ \hline
(\locvecO,\locvecO) & (1,0,0,0) & \dfrac{C_f^2}{(C_s+1)^2(2C_s+1)} \\ 
(\locvecO,\locvecx) & (0,1,0,0) & -\dfrac{C_sC_f}{(C_s+1)^2(2C_s+1)}\\
(\locvecx,\locvecO) & (0,0,1,0) & -\dfrac{C_sC_f}{(C_s+1)^2(2C_s+1)}\\
(\locvecx,\locvecx) & (0,0,0,1) & \dfrac{C_s^2}{(C_s+1)^2(2C_s+1)}.
\end{array}
\label{CountGenericChangeTable}
\end{align}
The four rows in \eqref{CountGenericChangeTable} are obtained by substituting $Z=(1,0,0,0)$, $Z=(0,1,0,0)$, $Z=(0,0,1,0)$, and $Z=(0,0,0,1)$ into the right hand side of \eqref{epicount}, and then subtracting \eqref{CountBaseline}. For example, the $(\locvecO,\locvecx)$ row is
\begin{align*}
F(0,1,0,0) - F(0,0,0,0) & = \frac{2C_f^2-C_f}{2(C_s+1)^2}
+\frac{C_f}{2(C_s+1)^2(2C_s+1)}
-\frac{C_f^2}{(C_s+1)^2} \notag\\
& = -\frac{C_sC_f}{(C_s+1)^2(2C_s+1)}.
\end{align*}

For each $\Delta$, sampling without replacement gives
\begin{align}
\ExptWR{\sigma}{\Card{\mathcal{L}_0}} & =\frac{C_s(C_s-1)}{D-1}, & \ExptWR{\sigma}{\Card{\mathcal{L}_1}} & =\frac{C_sC_f}{D-1}, \notag \\
\ExptWR{\sigma}{\Card{\mathcal{H}_0}} & =\frac{C_sC_f}{D-1}, & \ExptWR{\sigma}{\Card{\mathcal{H}_1}} & =\frac{C_f(C_f-1)}{D-1}. \label{CountGenericExpectedPairs}
\end{align}
The same calculation as in \eqref{expectedTwoAdjacent}, with $f$ replaced by $C_s+C_f$ gives
\begin{align*}
\ExptWR{\sigma}{R_{\Delta}(R_{\Delta} - 1)} = O(D^{-2}).
\end{align*}
Therefore $\Prob{R_\Delta\geq2} = O(D^{-2})$. When $R_{\Delta} = 0$, the change in $F$ is zero by \eqref{CountBaseline}. When $R_{\Delta}=1$, it is the corresponding entry of \eqref{CountGenericChangeTable}. When $R_{\Delta} \geq 2$, $F$ is bounded because $C_s$ and $C_f$ are constant. Multiplying the four entries in \eqref{CountGenericChangeTable} by the four expectations in \eqref{CountGenericExpectedPairs} and adding them gives
\begin{align}
\CovWR{\sigma,\pi}{\theta_s}{\theta_{s+\Delta}} & =\frac{1}{(C_s+1)^2(2C_s+1)(D-1)}
\left\{C_f^2C_s(C_s-1) \right. \notag \\
& \hspace*{1.3cm}\left. -2C_s^2C_f^2 + C_s^2C_f(C_f-1)\right\}+O(D^{-2}) \notag \\
& = -\frac{C_sC_f(C_s+C_f)}
{(C_s+1)^2(2C_s+1)(D-1)} + O(D^{-2}). \label{countSparseGenericRaw}
\end{align}

Now suppose that $D$ is even and $\Delta = D/2$. The tuple $(x_i, x_{i+D/2})$ is accompanied by $(x_{i+D/2}, x_i)$. Let $Q_{\Delta}$ be the number of pairs of positions $\{i, i+D/2\}$ for which both entries are $\locvecO$ or $\locvecx$. When $Q_{\Delta}=1$, the possible values of $Z$ and their changes are
\begin{align}
\begin{array}{c|c|c}
\text{entries in the pair} & Z & F(Z)-F(0,0,0,0) \\ \hline
(\locvecO,\locvecO) &(2,0,0,0) & \dfrac{C_f^2}{C_s(C_s+1)^2} \\
(\locvecO,\locvecx) & (0,1,1,0) & -\dfrac{2C_sC_f+1}{(C_s+1)^2(2C_s+1)} \\
(\locvecx,\locvecx) & (0,0,0,2) & \dfrac{2C_s^2}{(C_s+1)^2(2C_s+1)}.
\end{array}
\label{CountDoverTwoChangeTable}
\end{align}
For example, a pair containing $\locvecO$ and $\locvecx$ gives the two tuples $(\locvecO,\locvecx)$ and $(\locvecx,\locvecO)$, so $Z=(0,1,1,0)$. Substitution into \eqref{epicount} gives
\begin{align*}
F(0,1,1,0) - F(0,0,0,0) & = \frac{2C_f^2-2C_f}{2(C_s+1)^2} + \frac{2C_f-2}{2(C_s+1)^2(2C_s+1)} -\frac{C_f^2}{(C_s+1)^2} \\
& = -\frac{2C_sC_f+1}{(C_s+1)^2(2C_s+1)}.
\end{align*}

There are $D/2$ pairs of positions of the form $\{i, i+D/2\}$. Sampling the two entries without replacement gives respective expectations
\begin{align}
\frac{C_s(C_s-1)}{2(D-1)},~~~~
\frac{C_sC_f}{D-1},~~~~
\frac{C_f(C_f-1)}{2(D-1)}
\label{CountDoverTwoExpectedPairs}
\end{align}
for $(\locvecO,\locvecO)$, $(\locvecO,\locvecx)$, and $(\locvecx,\locvecx)$, respectively. The calculation in \eqref{expectedTwoDoverTwo} with $f$ replaced by $C_s+C_f$ gives
\begin{align*}
\ExptWR{\sigma}{Q_{\Delta} (Q_{\Delta}-1)}=O(D^{-2}).
\end{align*}
Thus the cases with $Q_{\Delta} \geq 2$ affect the expectation only by $O(D^{-2})$. Multiplying the entries in \eqref{CountDoverTwoChangeTable} by the three expectations in \eqref{CountDoverTwoExpectedPairs} and adding them gives
\begin{align}
\CovWR{\sigma,\pi}{\theta_s}{\theta_{s+\Delta}} & = -\frac{C_f(2C_s+C_f)}{2(C_s+1)(2C_s+1)(D-1)}+O(D^{-2}).
\label{countSparseSpecialRaw}
\end{align}

Lemma~\ref{emptylemma} gives the one shift variance
\begin{align}
V_{\mathrm{ind}} = \frac{C_sC_f(C_s+C_f+1)}{(C_s+1)^2(C_s+2)}.
\label{CountIndividualVariance}
\end{align}
Dividing the leading term of \eqref{countSparseGenericRaw} by \eqref{CountIndividualVariance} gives
\begin{align}
\CovWR{\sigma,\pi}{\theta_s}{\theta_{s+\Delta}} = -V_{\mathrm{ind}}\frac{(C_s+2)(C_s+C_f)}{(2C_s+1)(C_s+C_f+1)(D-1)} + O(D^{-2}),
\label{CountSparseGenericNormalized}
\end{align}
while dividing \eqref{countSparseSpecialRaw} by \eqref{CountIndividualVariance} gives
\begin{align}
\CovWR{\sigma,\pi}{\theta_s}{\theta_{s+\Delta}} = -V_{\mathrm{ind}} \frac{(C_s+1)(C_s+2)}{2C_s(2C_s+1)}\frac{2C_s+C_f}{(C_s+C_f+1)(D-1)}+O(D^{-2}). \label{CountSparseDoverTwoNormalized}
\end{align}
Equations \eqref{CountSparseGenericNormalized} and \eqref{CountSparseDoverTwoNormalized} give the next two expressions in Lemma~\ref{lastlemA} when $C_s$ and $C_f$ are fixed.

Finally, we now suppose that both $C_s/D$ and $C_f/D$ are bounded away from zero. We continue to use $Z$ from \eqref{CountZdefinition} and $F$ from \eqref{epicount}. The expectation of $Z$ is
\begin{align}
\mu = \ExptWR{\sigma}{Z}
=\frac{1}{D-1} \left(C_s(C_s-1),~C_sC_f,~C_sC_f,~C_f(C_f-1)\right). \label{CountMeanVector}
\end{align}

We use $P(A_1,\hdots,A_r)$ to denote the probability that a sequence of $r$ draws without replacement from the location vector gives $A_1,\ldots,A_r$ in this order, where every $A_i$ is $\locvecO$, $\locvecx$, or $\locvecm$. For example,
\begin{align*}
P(\locvecO,\locvecx) & = \frac{C_sC_f}{D(D-1)} 
\end{align*}

For any two entries $A,B \in \{\locvecO,\locvecx,\locvecm\}$, let
\begin{align}
Z_{AB}=\sum_{i=1}^D\indic{(x_i,x_{i+\Delta})=(A,B)}.
\label{CountZABdefinition}
\end{align}
Thus $Z_{AB}$ is the number of times $(A,B)$ occurs among the tuples $(x_i,x_{i+\Delta})$.

Suppose first that $D$ is odd or $\Delta \neq D/2$. There are $D^2$ terms after expanding $Z_{AB} Z_{CE}$ in \eqref{CountZABdefinition}. If the two starting positions are equal, there are $D$ terms, and they are nonzero only if $A = C$ and $B = E$. If the second starting position is $\Delta$ after the first, there are $D$ terms, and the common entry requires $B = C$. If the first starting position is $\Delta$ after the second, there are $D$ terms, and the common entry requires $E = A$. The remaining $D(D-3)$ terms use four distinct positions. Adding these four cases gives
\begin{align}
\ExptWR{\sigma}{Z_{AB}Z_{CE}} & = D\indic{A=C,B=E} P(A,B) + D\indic{B=C}P(A,B,E) \notag\\
&~~+D\indic{E=A}P(C,A,B) +D(D-3)P(A,B,C,E).
\label{CountOtherSecondMoment}
\end{align}
Now suppose that $D$ is even and $\Delta=D/2$. There are $D$ terms with the same starting position. There are another $D$ terms in which $(x_j,x_{j+\Delta})=(x_{i+\Delta},x_i)$, which requires $A=E$ and $B=C$. The remaining $D(D-2)$ terms have four distinct positions. Therefore
\begin{align}
\ExptWR{\sigma}{Z_{AB}Z_{CE}} & =  D \indic{A=C,B=E}P(A,B) +D\indic{A=E,B=C}P(A,B) \notag\\
&~~~+D(D-2)P(A,B,C,E).
\label{CountDoverTwoSecondMoment}
\end{align}
Equations \eqref{CountOtherSecondMoment} and \eqref{CountDoverTwoSecondMoment}, together with their corresponding draw probabilities, give the sixteen second moments of the four entries in \eqref{CountZdefinition}. For example, when $D$ is odd or $\Delta\neq D/2$,
\begin{align*}
\ExptWR{\sigma}{\Card{\mathcal L_0} \Card{\mathcal{L}_1}} = D P(\locvecO,\locvecO,\locvecx)+D(D-3)P(\locvecO,\locvecO,\locvecO,\locvecx),
\end{align*}
because the second $\locvecO$ in $(\locvecO,\locvecO)$ can be the first $\locvecO$ in $(\locvecO,\locvecx)$. When $D$ is even and $\Delta=D/2$,
\begin{align*}
\ExptWR{\sigma}{\Card{\mathcal {L}_1}\Card{\mathcal H_0}} =D P(\locvecO,\locvecx) + D(D-2) P(\locvecO,\locvecx,\locvecx,\locvecO),
\end{align*}
because $(\locvecO,\locvecx)$ and $(\locvecx,\locvecO)$ can be the two tuples formed from the same two positions.

Let $W=Z-\mu$. Taylor expansion of $F(\mu+W)$ and taking expectations give
\begin{align}
\ExptWR{\sigma}{F(Z)} & = F(\mu) + \frac{1}{2} \sum_{r,q=0}^3 F_{rq}(\mu) \CovWR{\sigma}{Z_r}{Z_q} \notag \\
& ~~~+ \frac16\sum_{r,q,v=0}^3 F_{rqv}(\mu)\ExptWR{\sigma}{W_rW_qW_v} + R_D.
\label{CountTaylorExpansion}
\end{align}
Here $F_{rq}$ and $F_{rqv}$ denote the respective partial derivatives. Because $\Card{\mathcal{L}_0}\leq C_s$, every denominator in \eqref{epicount} is at least of order $C_s$. Since $C_s/D$ is bounded away from zero, the derivatives of orders two, three, and four are $O(D^{-2})$, $O(D^{-3})$, and $O(D^{-4})$, respectively. Expanding products of three and four sums of the form \eqref{CountZABdefinition}, and separating the common index cases in the same way as in \eqref{CountOtherSecondMoment} and \eqref{CountDoverTwoSecondMoment}, gives
\begin{align}
\ExptWR{\sigma}{W_rW_qW_v}=O(D), ~~~\ExptWR{\sigma}{\|W\|^4}=O(D^2). \label{CountHigherMomentOrders}
\end{align}
It follows from \eqref{CountHigherMomentOrders} that the third order term in \eqref{CountTaylorExpansion} is $O(D^{-2})$. The fourth order Taylor remainder satisfies $R_D = O(D^{-2})$ for the same reason. Thus, up to $O(D^{-2})$, only $F(\mu)$ and the covariance term remain in \eqref{CountTaylorExpansion}.

Now, we can use a computer to find \eqref{CountTaylorExpansion}. First, compute every covariance term as
\begin{align*}
\CovWR{\sigma}{Z_{AB}}{Z_{CE}} = \ExptWR{\sigma}{Z_{AB}Z_{CE}} -\ExptWR{\sigma}{Z_{AB}}\ExptWR{\sigma}{Z_{CE}}, 
\end{align*}
using \eqref{CountMeanVector} with \eqref{CountOtherSecondMoment} or \eqref{CountDoverTwoSecondMoment}. Substitute these covariances into the second term of \eqref{CountTaylorExpansion}, use $C_n=D-C_s-C_f$, and collect the terms of order $D^{-1}$. This gives
\begin{align}
\CovWR{\sigma,\pi}{\theta_s}{\theta_{s+\Delta}} = \left\{\begin{array}{l l}
 -\dfrac{C_f\{C_s(2D-C_s)+C_f(D-C_s)\}}{C_s^2(2D-C_s)^2} +O(D^{-2}) , &\begin{array}{l}D\text{ even and}\\
\Delta=D/2,
\end{array}\\
-\dfrac{C_f(C_s+C_f)}{C_s^2(2D-C_s)} +O(D^{-2}), &\text{otherwise.}
\end{array}\right.
\label{CountDenseRawExpansion}
\end{align}
Equation~\eqref{CountDenseRawExpansion} follows by combining \eqref{CountMeanVector}, the respective draw probabilities, the second moments \eqref{CountOtherSecondMoment} and \eqref{CountDoverTwoSecondMoment}, and the expansion \eqref{CountTaylorExpansion}, and noting that $\CovWR{\sigma,\pi}{\theta_s}{\theta_{s+\Delta}} = \ExptWR{\sigma}{F(Z)} -\left(\frac{C_f}{C_s+1}\right)^2$.

Next, the variance in \eqref{CountIndividualVariance} satisfies
\begin{align}
V_{\mathrm{ind}} & = \frac{C_f(C_s+C_f)}{C_s^2}+O(D^{-1}), \label{CountDenseVarianceApproximation} \\
\frac{V_{\mathrm{ind}}}{C_s+C_f} & =\frac{C_f}{C_s^2}+O(D^{-2}). \label{CountDenseScaledVarianceApproximation}
\end{align}
Substituting \eqref{CountDenseVarianceApproximation} and \eqref{CountDenseScaledVarianceApproximation} into \eqref{CountDenseRawExpansion} gives
\begin{align*}
\CovWR{\sigma,\pi}{\theta_s}{\theta_{s+\Delta}}
=\left\{\begin{array}{l l}
-V_{\mathrm{ind}}
\dfrac{C_s(2D-C_s)+C_f(D-C_s)}
{(C_s+C_f)(2D-C_s)^2} +O(D^{-2}), &\begin{array}{l} D\text{ even and}\\
\Delta=D/2,\end{array}\\ 
-V_{\mathrm{ind}}\dfrac{1}{2D-C_s}+O(D^{-2}), & \text{otherwise.}
\end{array}\right.
\end{align*}

\end{proof}

\subsection{Additional Experiments to Verify Variance Approximation}\label{appendixBExpts}

Since the Minner estimator depends on $C_s=f_2$ and $C_f=f_1-a$, we investigate edge cases when $f_1-a, f_2$ are small relative to $D$, as $D$ increases. All experiments are done over 1,000 iterations (an order of magnitude less than the plots in Appendix~\ref{appendix_a5_EXPT}, since dealing with large $K$ means more expensive permutations). We can see that the theoretical variance approximations match the empirical MSEs well, as in Figures~\ref{appendixBfig1}-\ref{appendixBfig6}.

In the legend of these plots, {\tt Counts MSE} denotes the empirical MSE of the Minner count estimator, {\tt Counts Var} denotes the theoretical variance of the Minner count estimator, {\tt CP Counts MSE} denotes the empirical MSE of the count estimator with circulant permutations, {\tt CP Counts Var (S)} denotes the theoretical variance approximation when $\frac{C_s+C_f}{D} \equiv \frac{f_1+f_2-a}{D} \rightarrow 0$, and {\tt CP Counts Var (NS)} denotes the theoretical variance approximation when $\frac{f_1+f_2-a}{D}$ does not tend to zero. 

Since we are looking at counts when $\frac{f_1+f_2-a}{D}$ is small, {\tt CP Counts Var (S)} should be a better approximation. However, we find that both {\tt CP Counts Var (S)} and {\tt CP Counts Var (NS)} are equally good estimators as they are indistinguishable from each other.

Unlike the experiments in Appendix~\ref{appendix_a5_EXPT}, where the theoretical variance curves are almost indistinguishable from the empirical MSEs, we do not expect the count approximations here to overlap perfectly for moderate $D$, especially when $\frac{f_1+f_2-a}{D}$ is small, given that we use a log scale for the $y$ axes in these plots. However, the approximation is still accurate for practical comparisons. In particular, {\tt Counts Var} gives a useful reference for the amount of Monte Carlo discrepancy expected from 1,000 iterations.

\begin{figure}[h]
\begin{center}
\includegraphics[width=0.8\linewidth]{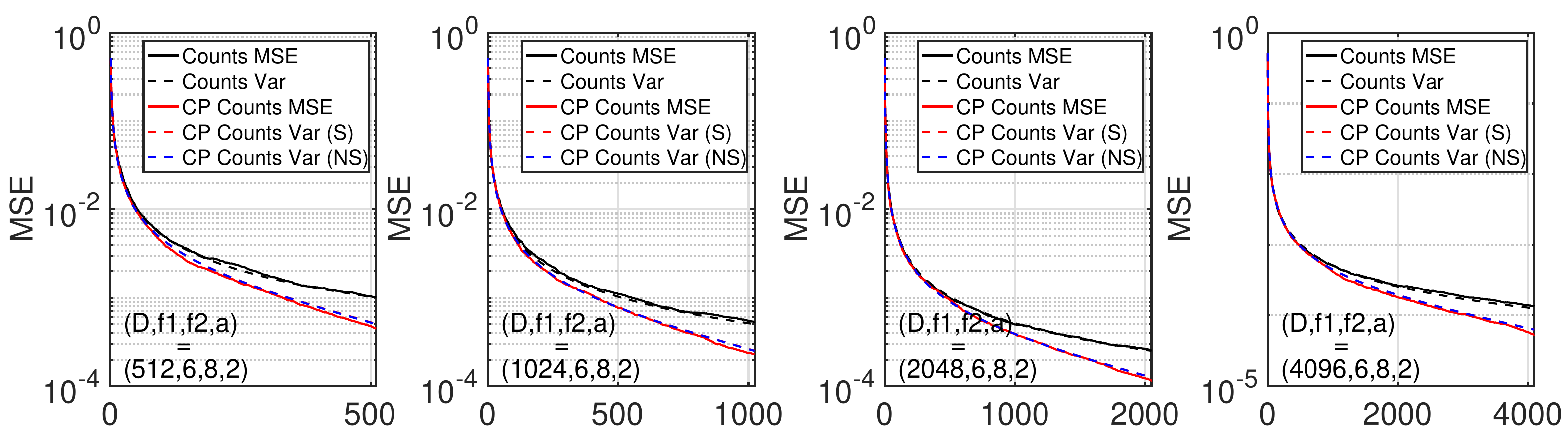}
\end{center}
\caption{Empirical MSE vs theoretical approximate variance in Theorem~\ref{countThm} at edge cases when $f_i$ are small respective to $D$. We use $(f_1,f_2,a) = (6,8,2)$. 
\label{appendixBfig1}}
\end{figure}

\begin{figure*}[h]
\begin{center}
\includegraphics[width=0.8\linewidth]{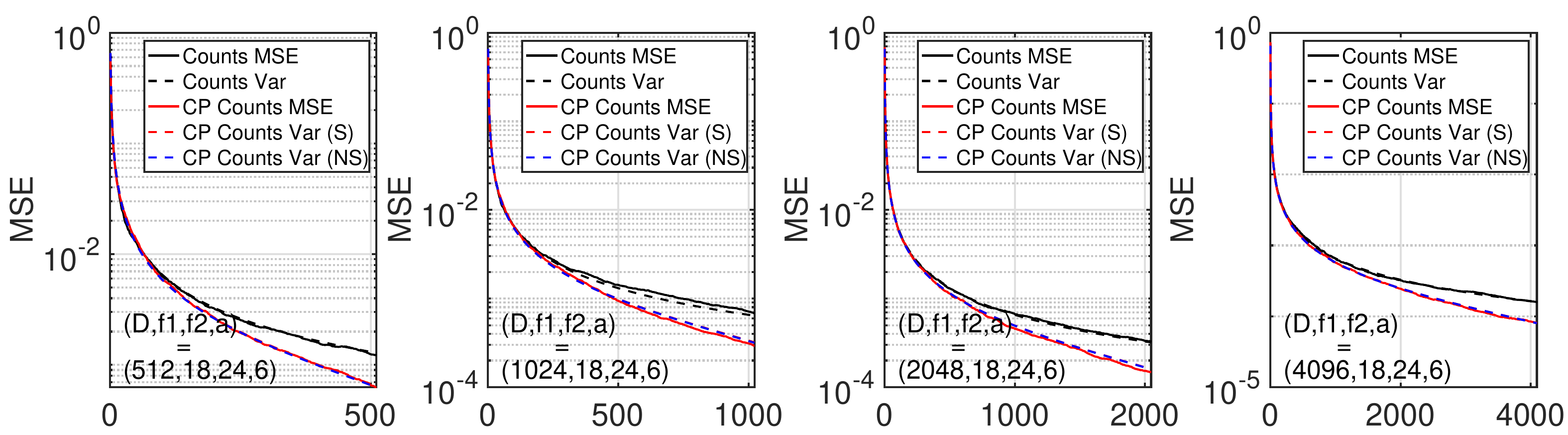}
\end{center}
\caption{Empirical MSE vs theoretical approximate variance in Theorem~\ref{countThm} at edge cases when $f_i$ are small respective to $D$. We use $(f_1,f_2,a) = (18,24,6)$. 
\label{appendixBfig2}}
\end{figure*}

\begin{figure*}[h]
\begin{center}
\includegraphics[width=0.8\linewidth]{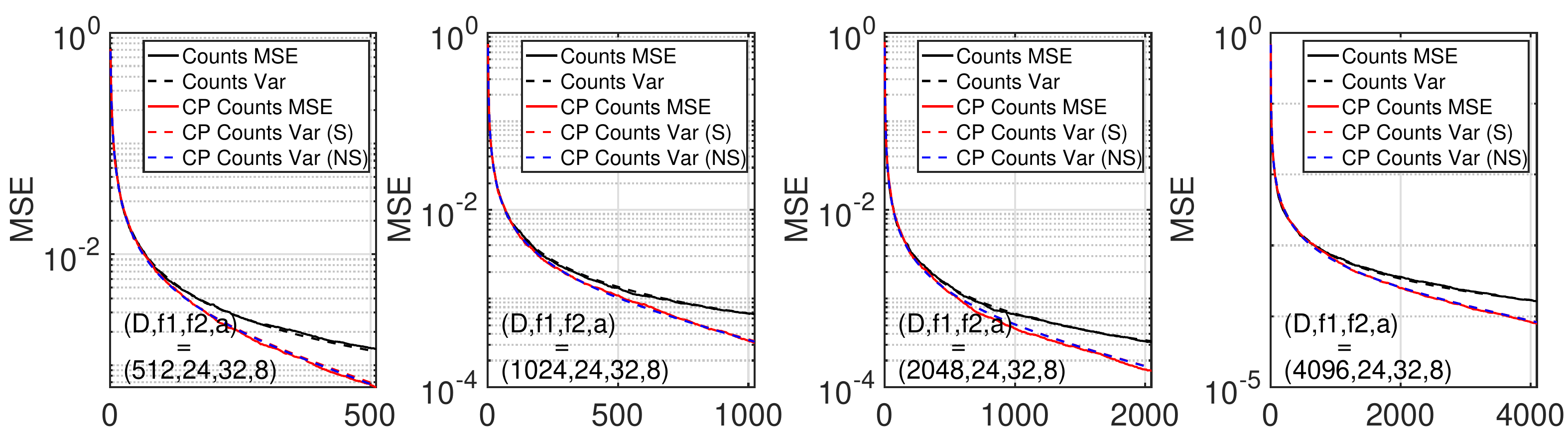}
\end{center}
\caption{Empirical MSE vs theoretical approximate variance in Theorem~\ref{countThm} at edge cases when $f_i$ are small respective to $D$. We use $(f_1,f_2,a) = (24,32,8)$. 
\label{appendixBfig4}}
\end{figure*}

\begin{figure*}[h]
\begin{center}
\includegraphics[width=0.8\linewidth]{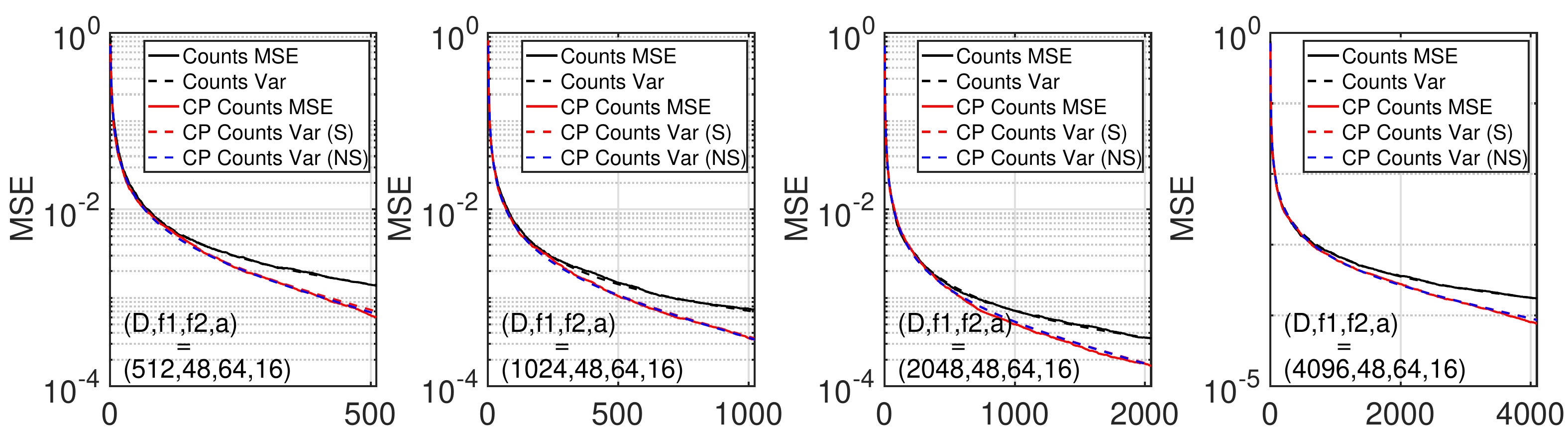}
\end{center}
\caption{Empirical MSE vs theoretical approximate variance in Theorem~\ref{countThm} at edge cases when $f_i$ are small respective to $D$. We use $(f_1,f_2,a) = (48,64,16)$. 
\label{appendixBfig5}}
\end{figure*}

\begin{figure*}[h]
\begin{center}
\includegraphics[width=0.8\linewidth]{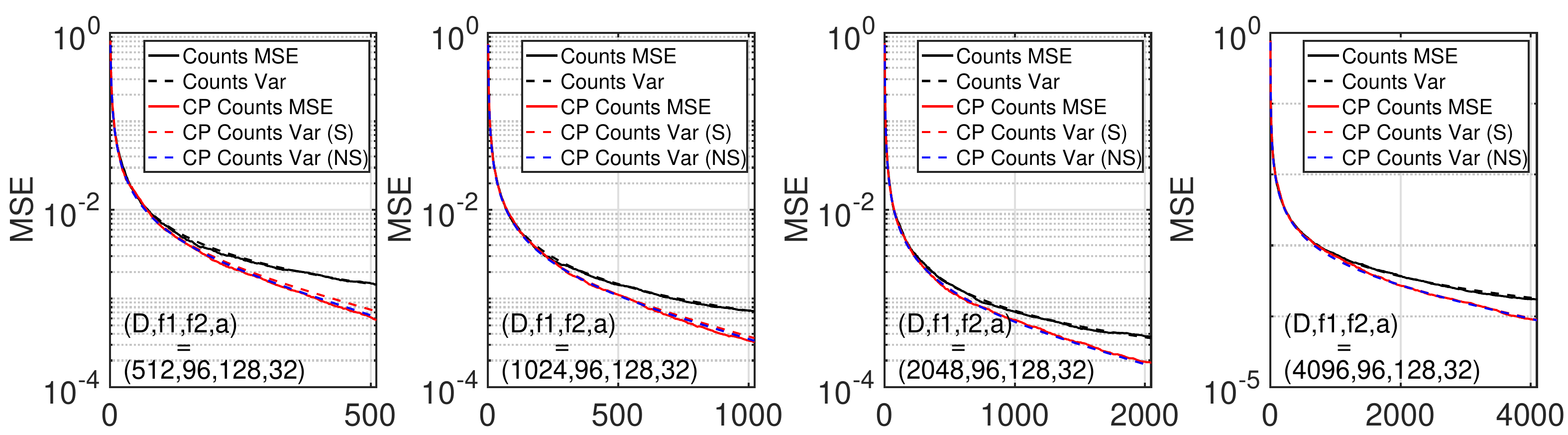}
\end{center}
\caption{Empirical MSE vs theoretical approximate variance in Theorem~\ref{countThm} at edge cases when $f_i$ are small respective to $D$. We use $(f_1,f_2,a) = (96,128,32)$. 
\label{appendixBfig6}}
\end{figure*}

\begin{figure*}[h]
\begin{center}
\includegraphics[width=0.8\linewidth]{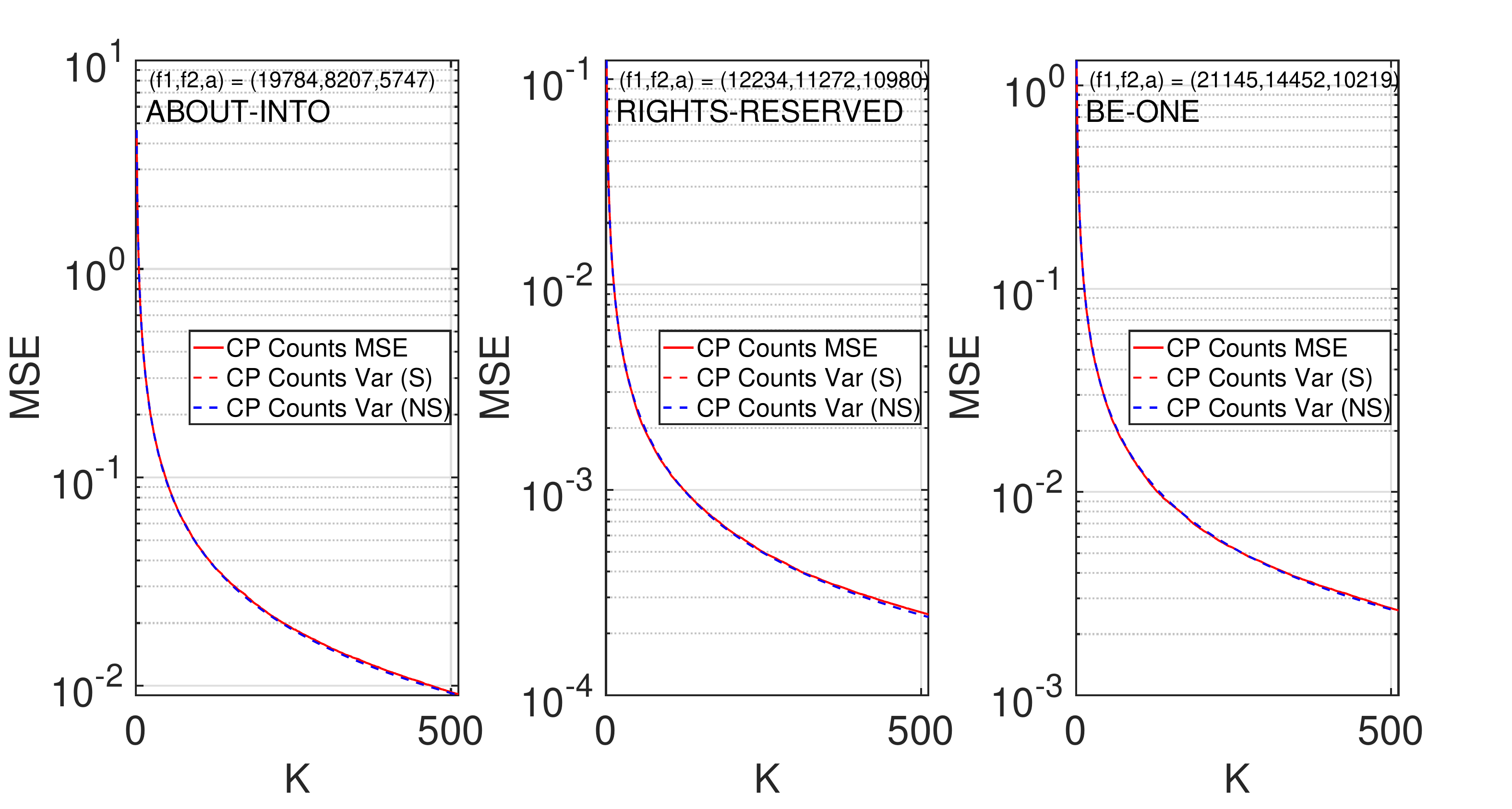}

\includegraphics[width=0.8\linewidth]{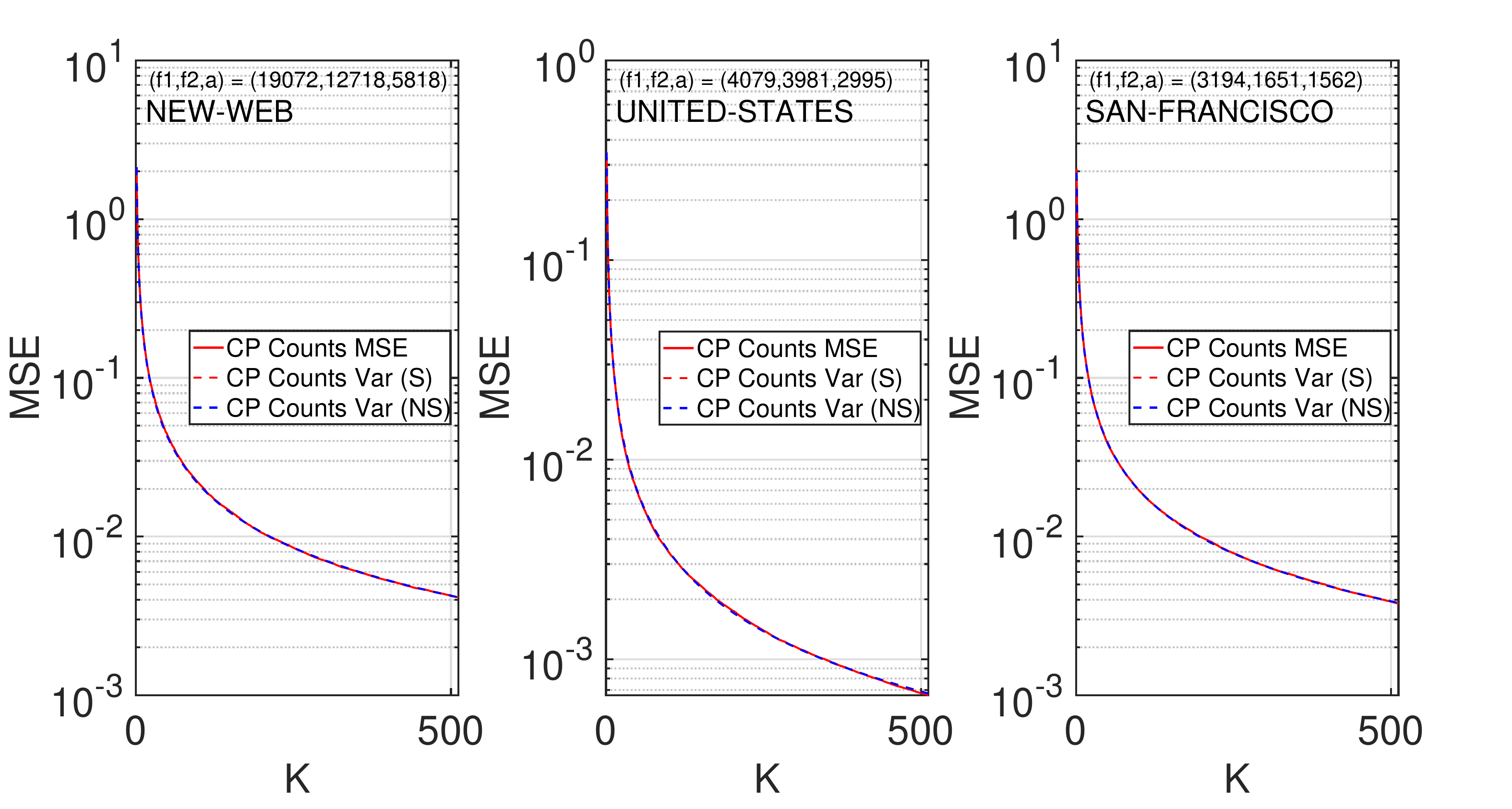}

\includegraphics[width=0.8\linewidth]{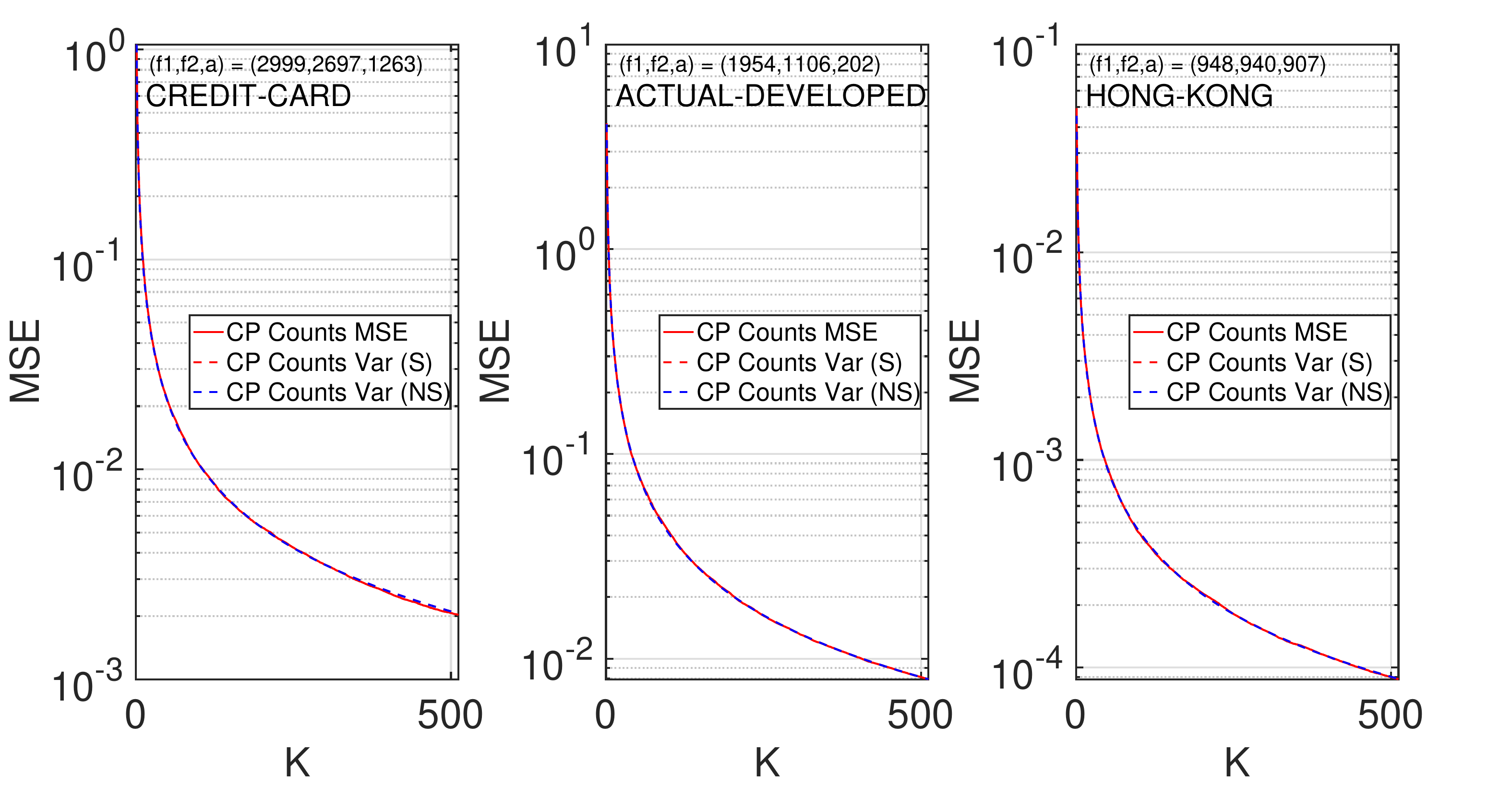}
\end{center}
\caption{Empirical MSE vs theoretical approximate variance in Theorem~\ref{countThm} for selected word pairs. Both theoretical approximations and empirical MSE lines overlap perfectly. \label{words_data_M_CMH}}
\end{figure*}

We also plot the theoretical approximation and  the empirical MSE for select word pairs from the words dataset \citep{li2010b} in Figure~\ref{words_data_M_CMH}. 

Both the theoretical approximate variance curves overlaps with the empirical MSEs for all word pairs, even the {\tt HONG}-{\tt KONG} word pair, where $\frac{f_1-a}{D} = \frac{41}{65536}\approx 6.3 \times 10^{-4}$ is small.

\clearpage

\newpage
\section*{NeurIPS Paper Checklist}

\begin{enumerate}

\item {\bf Claims}
    \item[] Question: Do the main claims made in the abstract and introduction accurately reflect the paper's contributions and scope?
    \item[] Answer: \answerYes{}
    \item[] Justification: Our abstract claims a new approach to approximate the variance of MinHash under circulant permutations. We demonstrate this in Section 2. We also claim that we describe how to apply our approach to different MinHash estimators. We describe how to do so in Section 4 for asymmetric hashing (Minner estimator). 

\item {\bf Limitations}
    \item[] Question: Does the paper discuss the limitations of the work performed by the authors?
    \item[] Answer: \answerYes{} % Replace by \answerYes{}, \answerNo{}, or \answerNA{}.
    \item[] Justification: The limitations are in the last section of the paper. 

\item {\bf Theory assumptions and proofs}
    \item[] Question: For each theoretical result, does the paper provide the full set of assumptions and a complete (and correct) proof?
    \item[] Answer: \answerYes{} % Replace by \answerYes{}, \answerNo{}, or \answerNA{}.
    \item[] Justification: A full and complete proof for all lemmas, theorems, and corollaries are in the appendix. The assumption that $D$ is large is clearly stated. 

    \item {\bf Experimental result reproducibility}
    \item[] Question: Does the paper fully disclose all the information needed to reproduce the main experimental results of the paper to the extent that it affects the main claims and/or conclusions of the paper (regardless of whether the code and data are provided or not)?
    \item[] Answer: \answerYes{} % Replace by \answerYes{}, \answerNo{}, or \answerNA{}.
    \item[] Justification: Our experiments verify if our approximate variances are close to the empirical MSEs of existing estimators (CMinHash, and circulant permutation applied to the Minner estimator). We review how circulant permutations work in Appendix A1, and give the approximate variances in Theorems 2 and 4.

\item {\bf Open access to data and code}
    \item[] Question: Does the paper provide open access to the data and code, with sufficient instructions to faithfully reproduce the main experimental results, as described in supplemental material?
    \item[] Answer: \answerYes{} % Replace by \answerYes{}, \answerNo{}, or \answerNA{}.
    \item[] Justification: Our data is mostly synthetic, and the words dataset is an existing dataset used by \cite{li2022minhash}. We describe our experiments clearly in Sections 2 and 3. Our experiments are mostly similar to \cite{li2022minhash}, hence code is not provided.

\item {\bf Experimental setting/details}
    \item[] Question: Does the paper specify all the training and test details (e.g., data splits, hyperparameters, how they were chosen, type of optimizer) necessary to understand the results?
    \item[] Answer: \answerYes{} % Replace by \answerYes{}, \answerNo{}, or \answerNA{}.
    \item[] Justification: Our data is mostly synthetic, and the words dataset is an existing dataset used by \cite{li2022minhash}. We describe our experiments clearly in Sections 2 and 3.

\item {\bf Experiment statistical significance}
    \item[] Question: Does the paper report error bars suitably and correctly defined or other appropriate information about the statistical significance of the experiments?
    \item[] Answer: \answerNA{} % Replace by \answerYes{}, \answerNo{}, or \answerNA{}.
    \item[] Justification: Because we are showing an approximation (which is a fixed value given $f_1,f_2,a,D$) of the actual variance of estimators, no error bars are needed. We compute the respective MSEs for MinHash and Minner estimates averaged over 10,000 iterations. If our approximation was correct, the respective average of the empirical MSEs should overlap with our theoretical variance. Any confidence intervals (or error bars) around the average of the 10,000 empirical MSEs would be of the form $\overline{\mathrm{MSE}} \pm \epsilon$. But if our approximate variance overlap the empirical MSEs, then the value of $\epsilon$ doesn't matter in this case (in fact, with 10,000 iterations, any $\epsilon$ would be sufficiently small to be mathematically negligible for visual comparison).

\item {\bf Experiments compute resources}
    \item[] Question: For each experiment, does the paper provide sufficient information on the computer resources (type of compute workers, memory, time of execution) needed to reproduce the experiments?
    \item[] Answer: \answerNo{} % Replace by \answerYes{}, \answerNo{}, or \answerNA{}.
    \item[] Justification: Our experiments mainly verify that our approximate variance matches empirical results, and no large scale compute resources are needed.  
    
\item {\bf Code of ethics}
    \item[] Question: Does the research conducted in the paper conform, in every respect, with the NeurIPS Code of Ethics \url{https://neurips.cc/public/EthicsGuidelines}?
    \item[] Answer: \answerYes{} % Replace by \answerYes{}, \answerNo{}, or \answerNA{}.
    \item[] Justification: We confirm to the NeurIPS Code of Ethics.

\item {\bf Broader impacts}
    \item[] Question: Does the paper discuss both potential positive societal impacts and negative societal impacts of the work performed?
    \item[] Answer: \answerNA{} % Replace by \answerYes{}, \answerNo{}, or \answerNA{}.
    \item[] Justification: This work serves to make the analysis of the variance of randomized algorithms involving circulant permutations easier.
    
\item {\bf Safeguards}
    \item[] Question: Does the paper describe safeguards that have been put in place for responsible release of data or models that have a high risk for misuse (e.g., pre-trained language models, image generators, or scraped datasets)?
    \item[] Answer: \answerNA{} % Replace by \answerYes{}, \answerNo{}, or \answerNA{}.
    \item[] Justification: This is a theoretical paper.

\item {\bf Licenses for existing assets}
    \item[] Question: Are the creators or original owners of assets (e.g., code, data, models), used in the paper, properly credited and are the license and terms of use explicitly mentioned and properly respected?
    \item[] Answer: \answerYes{} % Replace by \answerYes{}, \answerNo{}, or \answerNA{}.
    \item[] Justification: We cite the algorithm given and use the Words dataset in \cite{li2022minhash}. To the best of our knowledge, the words dataset first appears in \cite{li2005using}.

\item {\bf New assets}
    \item[] Question: Are new assets introduced in the paper well documented and is the documentation provided alongside the assets?
    \item[] Answer: \answerNA{} % Replace by \answerYes{}, \answerNo{}, or \answerNA{}.
    \item[] Justification: Our approach is theoretical.

\item {\bf Crowdsourcing and research with human subjects}
    \item[] Question: For crowdsourcing experiments and research with human subjects, does the paper include the full text of instructions given to participants and screenshots, if applicable, as well as details about compensation (if any)? 
    \item[] Answer: \answerNA{} % Replace by \answerYes{}, \answerNo{}, or \answerNA{}.
    \item[] Justification: We do not use human subjects.

\item {\bf Institutional review board (IRB) approvals or equivalent for research with human subjects}
    \item[] Question: Does the paper describe potential risks incurred by study participants, whether such risks were disclosed to the subjects, and whether Institutional Review Board (IRB) approvals (or an equivalent approval/review based on the requirements of your country or institution) were obtained?
    \item[] Answer: \answerNA{} % Replace by \answerYes{}, \answerNo{}, or \answerNA{}.
    \item[] Justification: We do not use human subjects.

\item {\bf Declaration of LLM usage}
    \item[] Question: Does the paper describe the usage of LLMs if it is an important, original, or non-standard component of the core methods in this research? Note that if the LLM is used only for writing, editing, or formatting purposes and does \emph{not} impact the core methodology, scientific rigor, or originality of the research, declaration is not required.
    %this research? 
    \item[] Answer: \answerYes{} % Replace by \answerYes{}, \answerNo{}, or \answerNA{}.
    \item[] Justification: Originally, we did not use an LLM for our research (apart from formatting, phrasing, etc). However, during our rebuttal period (also mentioned on Page 13), we asked 5.6 SOL for a combinatorial implementation of Minner, in order to better craft a rebuttal for our reviewers. It turned out that Li and Li's variance proof for CMinHash had a slight error. We included the counter example in the paper. 

\end{enumerate}

\end{document}